\documentclass[aos]{imsart}
\RequirePackage[utf8]{inputenc}
\RequirePackage[T1]{fontenc}
\RequirePackage{amsthm,amsmath,amsfonts,amssymb,mathtools}
\RequirePackage[numbers,sort&compress]{natbib}
\RequirePackage{xcolor}
\RequirePackage[colorlinks,citecolor=blue,urlcolor=blue]{hyperref}
\RequirePackage{graphicx,booktabs,subcaption,microtype}
\startlocaldefs
\numberwithin{equation}{section}
\hypersetup{bookmarksdepth=2}

\newcommand{\E}{\mathbb{E}}
\newcommand{\R}{\mathbb{R}}
\newcommand{\Var}{\mathrm{Var}}
\newcommand{\Cov}{\mathrm{Cov}}
\newcommand{\bA}{\mathbf{A}}
\DeclareMathOperator{\tr}{Tr}

\theoremstyle{plain}
\newtheorem{theorem}{Theorem}[section]
\newtheorem{proposition}[theorem]{Proposition}
\newtheorem{lemma}[theorem]{Lemma}
\newtheorem{corollary}[theorem]{Corollary}
\theoremstyle{remark}

\newtheorem{remark}[theorem]{Remark}
\newtheorem{example}[theorem]{Example}
\endlocaldefs
\begin{document}
\begin{frontmatter}
\title{Intrinsic Gradient Barriers in Pushforward Deep Generative Models}

\runtitle{Intrinsic Gradient Barriers}
\begin{aug}
\author[A]{\fnms{Edric}~\snm{Tam}\thanksref{T1}\ead[label=e1]{edrictam@stanford.edu}},
\author[B]{\fnms{Qingyang}~\snm{Liu}\ead[label=e2]{liuqi@wfu.edu}},
\author[A]{\fnms{Barbara}~\snm{Engelhardt}\ead[label=e3]{bengelhardt@stanford.edu}}
\and
\author[C]{\fnms{Ray}~\snm{Bai}\thanksref{T1}\ead[label=e4]{rbai2@gmu.edu}}
\address[A]{Department of Biomedical Data Science, Stanford University\printead[presep={,\ }]{e1,e3}}
\address[B]{Department of Statistical Sciences, Wake Forest University\printead[presep={,\ }]{e2}}
\address[C]{Department of Statistics, George Mason University\printead[presep={,\ }]{e4}}
\end{aug}
\thankstext{T1}{Corresponding authors.}

\begin{abstract}
Broad classes of deep generative models (DGMs) operate by learning to map a simple source distribution, such as a standard Gaussian, to a complex high-dimensional target distribution. Despite their success, DGMs often empirically struggle to accurately represent distributions with heavy tails or well-separated modes. In this work, we show that these difficulties reflect intrinsic properties of the underlying transport problem rather than particular neural network architectures or training procedures. We establish quantitative lower bounds on the gradients and Jacobians of pushforward mappings that transport log-concave source distributions to such targets. For heavy-tailed target distributions, the required gradient magnitude grows as the target develops increasingly large moments. For multimodal target distributions, the low-probability regions between well-separated modes similarly force large gradients in the transport map. We discuss how such large derivatives relate to trade-offs between expressivity and numerical stability. Our results hold under substantially weaker regularity conditions than previous analyses and remain stable under perturbations of the target distribution. Finally, we show that these obstructions can be mitigated when the goal is approximation of a given density rather than generative modeling. Specifically, we propose simple classes of bijective reparametrizations that can alleviate the difficult geometric structure and reduce the gradient burden placed on the learned generative map. Numerical experiments support our theoretical analysis.
\end{abstract}

\end{frontmatter}

\section{Introduction}
\label{sec:intro}

Deep generative models (DGMs) are foundational tools in modern artificial intelligence. Despite their success in tasks such as image generation, it has been empirically observed that DGMs often struggle when the target distribution is heavy-tailed \citep{pandey2024heavy, kim2023t} or multimodal \citep{salmona2022can,tanielian2020}. In particular, DGMs can grossly under-characterize tail events \citep{tam2025statistical} and fail to represent distinct modes \citep{khayatkhoei2018disconnected}. These phenomena are consequential in practice since DGMs are increasingly adopted in scientific domains where heavy-tailed or multimodal data are the norm rather than the exception \citep{ma2024generative, horvarth2026marketgen}. Heavy-tailed data arise in applications such as the climate sciences \citep{coles2001introduction, katz2002statistics}, where inadequate representation of extreme events can lead to under-estimation of extreme climate occurrences \citep{zscheischler2018future}, and finance \citep{bradley2003financial}, where underestimating tail probabilities can distort estimates of important risk measures such as Expected Shortfall \citep{embrechts2013modelling, mcneil2015quantitative}. Similarly, multimodal distributions arise in applications ranging from single-cell analysis, where missing modes can obscure biologically distinct cell populations \citep{marouf2020realistic}, to the modeling of heterogeneous populations, where omitting modes associated with under-represented subgroups can undermine the accuracy and fairness of downstream analyses \citep{marchesi2025generative}. 

\begin{figure*}[t]
    \centering
    \captionsetup{labelfont={color=black},textfont={color=black}}
    \includegraphics[width=.98\textwidth]{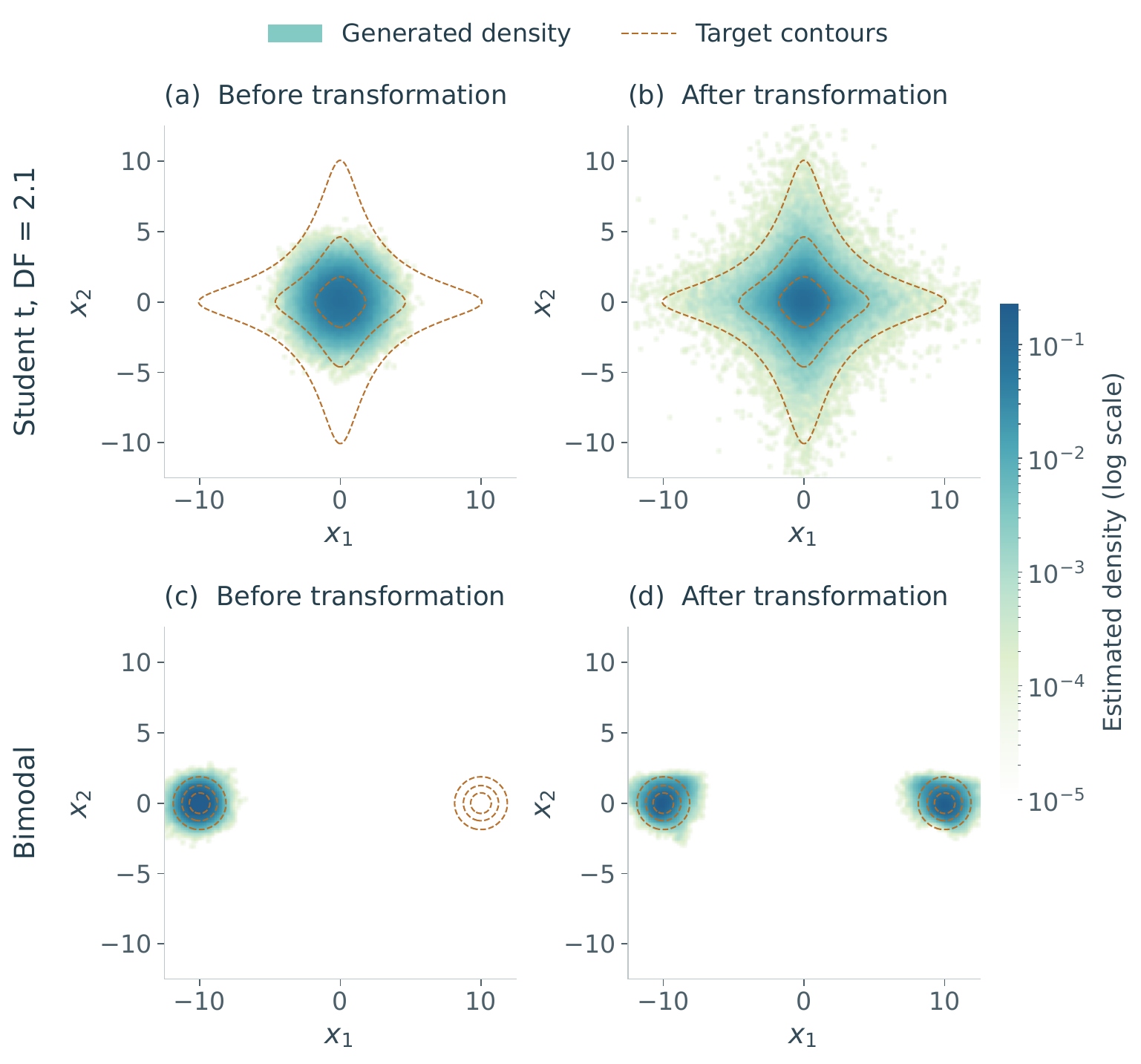}
    \caption{Results from using normalizing flows with a Gaussian source distribution to approximately sample from two target densities before (\textbf{left two panels}) and after (\textbf{right two panels}) applying bijective reparametrizations. The targets are a heavy-tailed Student-$t$ distribution (top) and a bimodal Gaussian mixture (bottom). Heatmaps show the densities estimated from the generated samples in the original coordinates, while dashed curves show the target contours.}
    \label{fig:density-fitting}
\end{figure*}

As DGMs have become more widely used, it is critical to understand \emph{why} DGMs can struggle to faithfully represent heavy-tailed and multimodal distributions. In this paper, we study the \emph{intrinsic} limitations of pushforward DGMs in representing such distributions. DGMs can be viewed through the lens of learning a map $T: \mathcal{Z} \to \mathcal{X}$ that pushes a simple source distribution $\nu$ (most commonly a Gaussian) on some latent space $\mathcal{Z}$ to a target distribution $\mu$ on the sample space $\mathcal{X}$. Prior work has examined limitations arising from global Lipschitz constraints on the neural network map $\widehat{T}$ used to approximate $T$ \citep{jaini2020tails, salmona2022can, tam2025statistical}. We ask a more fundamental question: what regularity constraints does the target distribution $\mu$ itself impose on \emph{any} transport map $T$ satisfying $T_\#\nu = \mu$? This perspective isolates the intrinsic obstruction in transporting $\nu$ to $\mu$, independently of the architecture or estimation procedure used to approximate the map. As a result, our analysis applies to a broad class of models, including various normalizing flow models whose induced transport maps need not be globally Lipschitz.

The left two panels of Figure~\ref{fig:density-fitting} highlight the intrinsic difficulties of implicitly sampling from two challenging target densities (dashed contours): a Student-$t$ with 2.1 degrees of freedom (top) and a bimodal Gaussian mixture (bottom). As shown by the heatmaps of the approximate samples, normalizing flows under a Gaussian source distribution struggle significantly with these targets. Specifically, they exhibit severe tail underestimation for the Student-$t$ density in Figure~\ref{fig:density-fitting}(a) and complete omission of one of the modes for the Gaussian mixture in Figure~\ref{fig:density-fitting}(c).  The experimental details for these simulations are provided in Section~\ref{supp:figure-one-fitting} of the Supplement. 

For our theoretical analysis of heavy-tailed targets, we leverage Poincar\'e-type inequalities to show that under the predominantly used Gaussian source measures, pushforward maps to targets with large centered moments must have large gradient components and Jacobians. We further generalize these results to log-concave and strongly log-concave source distributions. For multimodal target distributions with well-separated modes, we use isoperimetric and coarea inequalities to show that large gradients are a mathematical necessity, even when the target's moments remain bounded. It is well-known that, for invertible generative models, large transport derivatives can amplify numerical errors \citep{behrmann2021understanding}, leading to a tradeoff between expressivity and numerical stability. Our results thus unveil insights into why common DGMs can fundamentally struggle to learn transport maps for heavy-tailed and/or multimodal targets, \emph{irrespective} of their underlying neural network parameterization. 

Our analysis also suggests ways to mitigate these obstructions when the goal is to use such neural networks for approximation of a given target density rather than generative modeling. The target distribution in generative modeling is typically specified by training samples \citep{montanari2023sampling}. When the target density is known up to a normalizing constant, as is typical in variational inference \citep{rezende2015variational}, we can leverage this additional information to mitigate these obstructions. The key idea is to use known bijective reparametrizations to absorb part of the required derivative growth, leaving a more regular map to be learned. Although the resulting composite map remains subject to our lower bounds, this decomposition reduces the burden on the learned component of the map. 
The right two panels of Figure~\ref{fig:density-fitting} demonstrate the practical benefits of applying these bijective reparameterizations. Specifically, Figure~\ref{fig:density-fitting}(b) shows that the tail regions of the heavy-tailed Student-$t$ density are better represented, while Figure~\ref{fig:density-fitting}(d) demonstrates that both modes of the Gaussian mixture are correctly captured.
 
Our main contributions are summarized as follows:
\begin{enumerate}
\item For heavy-tailed targets, we derive lower bounds on transport gradients and Jacobians in terms of centered target moments. Our theory holds under Gaussian and log-concave source distributions without assuming a global Lipschitz bound. Importantly, these results characterize intrinsic limitations of the underlying transport problem, independent of any neural network architecture or training procedure.
\item For multimodal targets, we similarly show that low-probability valleys between modes force large transport gradients and Jacobians.
\item When the goal is to approximate a given density, we propose simple bijective reparametrizations for heavy-tailed and multimodal targets that reduce the derivative burden on the learned map. This allows us to alleviate the intrinsic gradient barriers in pushforward DGMs and more faithfully represent these challenging densities. 
\end{enumerate} 

The remainder of the paper is organized as follows. Section~\ref{sec:related}
reviews related work, and Section~\ref{sec:prelim} introduces the notation and setup.
Section~\ref{sec:gradient-bounds} uses functional inequalities to derive
gradient lower bounds in terms of target moments, while
Section~\ref{sec:multimodal} establishes localized Jacobian bounds for
multimodal targets. Section~\ref{sec:training_implications} discusses
practical implications for DGM architecture design and {numerical stability}.
Section~\ref{sec:transformations} develops {bijective reparametrizations} for
density approximation, and Section~\ref{sec:examples} presents numerical
illustrations. {\hypersetup{linkcolor=blue,filecolor=blue}Proofs, technical background, and additional experiments are provided in the Supplement.}

\section{Related Work}
\label{sec:related}

In this work, we focus on pushforward-based deep generative models, where the generated distribution is realized as the pushforward $T_\# \nu$ of a fixed source distribution $\nu$ by a learned map $T$. This perspective encompasses models that represent such maps explicitly with neural networks, including generative adversarial networks (GANs; \citep{goodfellow2020generative}), variational autoencoders (VAEs; \citep{kingma2014auto}), and normalizing flows \citep{rezende2015variational, dinh2016density}. It also includes continuous flow models, such as neural ordinary differential equations (ODEs) and related stochastic flow models \citep{chen2018neural, liu2022flow, albergo2023stochastic}, where a learned vector field is integrated over time to produce a transport map. Diffusion models \citep{song2020score, ho2020denoising, sohl2015deep} can also be treated within this pushforward framework after time discretization; the reverse sampling procedures can be viewed as pushforwards of a Gaussian measure on an augmented space \citep{salmona2022can, tam2025statistical}. Our results therefore apply to the resulting discretized sampling maps. Models that do not naturally specify a pushforward mechanism, such as energy-based models defined by unnormalized densities or large language models based on attention architectures, fall outside the scope of our analysis.

Several prior works have quantified the expressive power of DGMs under light-tailed/moments assumptions on the target \citep{lu2020universal, yang2022capacity}. However, there is growing empirical and methodological evidence that standard DGMs struggle to faithfully represent heavy-tailed targets \citep{ghane2025concentration, tam2025statistical, wiese2019copula, jaini2020tails}. These observations have motivated work aimed at mitigating such deficiencies, such as heavy-tailed diffusions \citep{pandey2024heavy}, Pareto GANs \citep{huster2021pareto}, and tail-adaptive normalizing flows \citep{hickling2024flexible}. For multimodal targets, \cite{salmona2022can} and \cite{tanielian2020} relate mode separation and approximation error to global Lipschitz constraints. Our multimodal analysis yields lower bounds on derivative moments under only local Lipschitz regularity.

The most closely related prior works are \citep{jaini2020tails} and \citep{salmona2022can}, who study heavy-tailed and multimodal targets under global Lipschitz assumptions. Related analyses of invertible networks and bi-Lipschitz flows also use global Lipschitz bounds \citep{hagemann2021stabilizing,verine2022}. Our results instead constrain gradients and Jacobians directly, without imposing a global bound on their size and allowing transport maps with unbounded derivatives. We require only Sobolev regularity for the moment bounds and local Lipschitz continuity for the multimodal bounds. These conditions are substantially weaker and cover a broader class of generative models, including normalizing flows and continuous flows whose transport maps need not be globally Lipschitz.

The main tools that we use in our analysis are functional and isoperimetric inequalities. These inequalities play a central role in probability theory, functional analysis, and the study of Markov chains \citep{bakry2013analysis}. In the context of sampling, such inequalities are often used to analyze the convergence of sampling algorithms such as Langevin dynamics \citep{chewi2025analysis}. While these inequalities are typically used to derive upper bounds on a Markov chain's distance to stationarity in the sampling literature, here we use them to \emph{lower-bound} the size of transport gradients and Jacobians.

{The gradient bounds derived in our analysis motivate changing the coordinates in which a target density is approximated. } Such coordinate transformations have a long history in the sampling literature. For example, measure transport methods \citep{marzouk2016sampling} and normalizing flows \citep{muller2019neural,gabrie2022adaptive} have been used to construct changes of coordinates that facilitate Monte Carlo sampling from complex target distributions. Gaussianization methods \citep{rosenblatt1952,chen2000gaussianization,goerg2015lambert} seek transformations that make a target distribution approximately Gaussian, whereas stereographic MCMC \citep{yang2024stereographic} uses stereographic projections to reformulate the sampling problem on a sphere. In this
article, we develop simple, bijective reparametrizations specifically for heavy-tailed and multimodal target densities.

\section{Preliminaries}
\label{sec:prelim}
\subsection{Setup and Definitions}
\label{subsec:notation}

For a matrix $\bA$, $\|\bA\|_F$ denotes its Frobenius norm, and $\|\bA\|_{\mathrm{op}}$ denotes its operator norm. For a vector $x \in \mathbb{R}^d$, $\|x\|_2$ denotes the Euclidean norm.  Generic positive constants are denoted by $C, c$, whose values may change from line to line.

Probability measures are denoted by Greek letters (e.g., $\nu, \mu$), with $\gamma$ reserved for the standard Gaussian measure. For a measurable map $T : \mathcal{Z} \to \mathcal{X}$ and a probability measure $\nu$ on $\mathcal{Z}$, we write $T_\# \nu$ for the pushforward measure induced by $T$. We denote by $J_T(x)$ the Jacobian of the map $T$ at $x$, and by $\nabla f$ the gradient of a scalar-valued function $f$. For vector-valued maps $T = (T_1, \dots, T_m)$, the Jacobian $J_T$ has entries $(J_T)_{ij} = \partial T_i / \partial x_j$. As is standard in the literature, when we write $\text{Var}_{\nu}(f)$ or $\E_\nu(f)$, we mean $\text{Var}(f(Z))$ and $\E(f(Z))$ for $Z \sim \nu$. Following a convention in high dimensional geometry, we call $\nu$ isotropic if it is centered and has identity covariance. We use $\kappa$ for Student-$t$ degrees of freedom, $\alpha$ for the Pareto shape parameter, and $\zeta$ for the Fr\'echet shape parameter.

Let $1 \le p < \infty$. For a full-dimensional log-concave source $\nu$, write $\Omega=\operatorname{int}(\operatorname{supp}\nu)$. A function $f$ belongs to the Sobolev space $W^{1,p}(\nu)$ if $f\in W^{1,p}_{\mathrm{loc}}(\Omega)$ and both $f$ and its weak gradient $\nabla f$ belong to $L^p(\nu)$. Weak derivatives are taken on $\Omega$; values outside $\Omega$ are irrelevant. All gradients and integrals in this work are understood in the above sense. See Section~\ref{append_technical} of the Supplement for background on weak gradients and Sobolev spaces.

A random variable $X \in \mathbb{R}$ is said to be heavy-tailed if $\E[e^{t|X|}] = \infty$ for all $t > 0$. A fundamental link between the $p$th moment of $X$ (if it exists) and its tail probabilities is given by the polynomial Markov (Chebyshev) inequality: for any $a > 0$, $
    \mathbb{P}(|X| \ge a)\le  \mathbb{E} |X|^p / a^p.$
This inequality shows that finite $p$th moments imply at least polynomial decay of the tail. We give more technical details on moments and tails in Section~\ref{append_technical} of the Supplement.

\subsection{Pushforward Maps and Deep Generative Models}
\label{subsec:existence_map}

Throughout, we adopt the push-forward perspective on DGMs. Let $\nu$ be a probability measure on a latent space $\mathcal Z \subseteq \mathbb R^d$ and let $\mu$ be a target probability measure on $\mathcal X \subseteq \mathbb R^m$. A measurable map $T:\mathcal Z \to \mathcal X$ is said to push $\nu$ forward to $\mu$ if $T_\# \nu = \mu$, i.e.,
\begin{align*}
\mu(A) = \nu\big(T^{-1}(A)\big) \quad \text{for all Borel } A \subseteq \mathcal X.
\end{align*}
When both $\mu$ and $\nu$ are non-atomic, the existence of such maps is guaranteed, although they are generally not unique \citep{villani2008optimal}.

In DGMs, one learns an approximation to such a pushforward map. That is, one considers a family $\{T_\theta : \theta \in \Theta\}$ of functions (typically involving neural networks) and seeks parameters $\hat\theta$ such that $(T_{\hat\theta})_\# \nu \approx \mu$ in an appropriate sense. The most typical setting in DGMs is to set $\nu = \mathcal{N}(0, I_d)$. One then generates approximate samples from $\mu$ by transforming a sample $Z \sim \nu$ into $X = T_{\hat\theta}(Z)$.

In models such as GANs and VAEs, the mapping $T_\theta$ is modeled explicitly by a neural network, where training involves minimizing a divergence (e.g., Jensen-Shannon or Kullback-Leibler) between $\mu$ and $(T_{\theta})_\# \nu$. In normalizing flows where the map is required to be invertible, neural networks with special coupling layers that ensure invertibility are used to model $T_\theta$. In dynamic models such as continuous flows, $T_\theta$ is realized implicitly through a composition of infinitesimal transformations (e.g., vector fields) that are modeled by neural networks. See \cite{tomczak2024deep} for a comprehensive overview of DGMs.

\section{Gradient Bounds for Deep Generative Models via Functional Inequalities} \label{sec:gradient-bounds}

Our goal is to formalize the principle that large moments of the target distribution necessarily require large gradients or Jacobians of the transport map, with an appropriate notion of size. To this end, it is natural to consider Poincar\'e inequalities, which relate the variance of the target to the gradient of the pushforward map. We begin with the most relevant case when the source distribution is Gaussian. Proofs of the exact-transport results in this section are delegated to Section~\ref{append_proofs} of the Supplement.

\subsection{Gaussian Source Distributions}

We say that $\nu$ satisfies a Poincar\'e inequality with constant $C_{\mathrm{PI}}(\nu) > 0$ if for all functions $f: \mathbb{R}^d \to \mathbb{R}$ belonging to $W^{1,2}(\nu)$ with finite variance, the following bound holds:
\begin{equation}
    \label{eq:pi}
    \text{Var}_{\nu}(f) \le C_{\mathrm{PI}}(\nu) \, \mathbb{E}_{\nu}[\|\nabla f\|_2^2].
\end{equation}
A Poincar\'e inequality implies that if $f(Z)$ has high variance (i.e., if $f_\# \nu$ has a large centered second moment), then the gradient $\nabla f$ must be large in expectation. In particular, the standard Gaussian measure $\mathcal{N}(0, I_d)$ satisfies (\ref{eq:pi}). Further background on the Poincar\'e constant is
provided in Section~\ref{supp:cheeger-poincare}
of the Supplement.

\begin{lemma}[Gaussian Poincar\'e inequality]
\label{thm:gaussian_poincare}
Let $\gamma$ denote the standard Gaussian measure on $\mathbb{R}^d$, and let
$f \in W^{1,2}(\gamma)$. Then
\begin{equation}  \label{eq:gaussian-pi-inequality}
\Var_{\gamma}(f)
 \le
C\int_{\mathbb{R}^d} \|\nabla f(x)\|_2^2\, d\gamma(x),
 \end{equation}
 where $C = 1$.
\end{lemma}

Intuitively, the Gaussian Poincar\'e inequality implies that producing a target distribution with large variance from a Gaussian source requires a pushforward map whose gradient has a large expected squared norm.

The space $W^{1,2}(\gamma)$ in Lemma \ref{thm:gaussian_poincare} only requires very weak regularity conditions (weak differentiability of the function and square-integrability of its gradient) that ensures expectations and gradients are well defined. This is substantially more general than global Lipschitz assumptions, which impose a uniform upper bound on the gradient. Notably, the inequality (\ref{eq:gaussian-pi-inequality}) is independent of the input dimension. It is a well-known fact that inequality (\ref{eq:gaussian-pi-inequality}) is tight, with equality achieved by affine functions \citep{klartag2025isoperimetric}.

Since DGMs concern mappings from a latent space to a high-dimensional target space, this motivates a multivariate extension of the Gaussian Poincar\'e inequality.

\begin{theorem}[Gaussian Poincar\'e inequality for vector-valued maps]
\label{lem:vector_poincare}
Let $Z\sim\mathcal N(0,I_d)$ and let $T:\mathbb R^d\to\mathbb R^m$ have components
$T=(T_1,\dots,T_m)$ with $T_j\in W^{1,2}(\gamma)$ for each $j$.
Let $J_T(z)\in\mathbb R^{m\times d}$ denote the Jacobian of $T$. Then:

\noindent\textbf{(i) Component-wise variance bounds}
For each $j=1,\dots,m$,
\begin{align*}
\Var\big(T_j(Z)\big)\le\E\big[\|\nabla T_j(Z)\|_2^2\big].
\end{align*}

\noindent\textbf{(ii) Covariance matrix bound}
For $\Cov(T(Z))\in\mathbb R^{m\times m}$, the covariance matrix of $T(Z)$, we have the following inequality:
\begin{align*}
\Cov\big(T(Z)\big)\preceq\E\big[\,J_T(Z)\,J_T(Z)^\top\big],
\end{align*}
where $\preceq$ denotes the Loewner order on the positive semidefinite matrices. Equivalently, for every $a\in\mathbb R^m$,
\begin{align*}
\Var\big(a^\top T(Z)\big)
&\le\E\big[\|J_T(Z)^\top a\|_2^2\big]
=a^\top \E\big[J_T(Z)J_T(Z)^\top\big]a.
\end{align*}
\noindent\textbf{(iii) Trace Bound}
In particular,
\begin{align*}
\tr\big[\Cov(T(Z))\big]
\le \E\big[\|J_T(Z)\|_F^2\big]
=\sum_{j=1}^m \E\big[\|\nabla T_j(Z)\|_2^2\big].
\end{align*}
\noindent\textbf{(iv) Entrywise covariance bound}
For $i,j\in\{1,\dots,m\}$,
\begin{align*}
|\Cov(T_i(Z),T_j(Z))|
&\le
\sqrt{\E\|\nabla T_i(Z)\|_2^2}\,
\sqrt{\E\|\nabla T_j(Z)\|_2^2}.
\end{align*}
\end{theorem}

The above inequalities relate the covariance of $T(Z)$ with the gradients and Jacobians of the map $T$. As an immediate consequence, if the target $\mu = T_\# \nu$ exhibits large centered second moments in its components, the Jacobian of the pushforward map $T$ must be correspondingly large. We discuss the practical implications of these results for DGMs in Section \ref{sec:training_implications}.

\subsection{$L_p$--$L_q$ Poincar\'e Inequalities}
\label{subsec:lp_lq_poincare}

The Gaussian Poincar\'e inequality lower bounds the expected gradient components and Jacobian by the second moments of the target. Given the close relationship between higher-order moments and heavy-tailed behavior, it is natural to seek analogous lower bounds on the gradients and Jacobian of the pushforward map in terms of higher-order moments of the target distribution. To achieve this, we use more recently developed mathematical tools known as $L_p$--$L_q$ Poincar\'e inequalities \citep{milman2009role, saumard2019isoperimetric, chewi22}. An $L_p$--$L_q$ Poincar\'e inequality for a probability measure $\nu$ on $\mathbb{R}^d$ asserts the existence of a positive constant $C_{p,q}<\infty$ such that
\begin{equation} \label{eq:poincare-inequality-lp-lq}
    \| f - \E_\nu f \|_{L^p(\nu)} \le C_{p,q}  \big\| {\lVert \nabla f\rVert}_2 \big\|_{L^q(\nu)},
\end{equation}
for all $f \in W^{1,q}(\nu) \cap L^p(\nu)$ for $p, q \in [1, \infty]$ with $q \geq p$. In other words, (\ref{eq:poincare-inequality-lp-lq}) quantifies how the $p$th centered moment of the target lower bounds the $L^q$-norm of the pushforward's gradient. The classical Poincar\'e inequality corresponds to the case $p=q=2$. For log-concave $\nu$, there is a natural notion of equivalence between $L_p$--$L_q$ Poincar\'e inequalities with different $p$'s and $q$'s \citep{milman2009role, chewi22}. For our purposes, we will focus on the case where $p = q$, and refer to these as $L_p$ Poincar\'e inequalities.

It is well known that the standard Gaussian measure satisfies an $L_p$ Poincar\'e inequality for all $2 \le p< \infty$, with constant $C_{p}$ that depends on $p$ but not on $d$.
\begin{lemma}[$L_p$ Gaussian Poincar\'e inequalities \citep{bogachev1998gaussian}]
\label{thm:gaussian_lp_poincare}
Let $\gamma$ denote the standard Gaussian measure on $\mathbb{R}^d$ and let
$Z\sim \mathcal N(0,I_d)$. For any $2\le p <\infty$ there exists a finite constant
$C_{p}$ depending only on $p$ such that for all $f\in W^{1,p}(\gamma)$,
$$
\|f-\E_\gamma f\|_{L^p(\gamma)}
\;\le\;
C_{p}\,\big\|{\lVert\nabla f\rVert}_2\big\|_{L^p(\gamma)}.
$$
\end{lemma}

Lemma \ref{thm:gaussian_lp_poincare} provides a lower bound on the norm of the gradient of a scalar-valued pushforward map in terms of the higher centered moments of the target distribution. Since DGMs have high-dimensional targets, we develop the corresponding multivariate bound below.
\begin{corollary}[$L_p$ Gaussian Poincar\'e for vector-valued maps]
\label{cor:vector_gaussian_lp_poincare}
Let $T=(T_1,\dots,T_m):\mathbb R^d\to\mathbb R^m$ with $T_j\in W^{1,p}(\gamma)$ for some $p\ge 2$.
We have
$$
\sum_{j=1}^m \|T_j-\E_\gamma T_j\|_{L^p(\gamma)}^p
\le
C'_{p}\,\E_\gamma \|J_T\|_F^p.
$$
where $C'_{p} >0$ is some constant that depends on $p$.
\end{corollary}

An immediate consequence of Corollary \ref{cor:vector_gaussian_lp_poincare} is that if the target distribution exhibits large centered $p$th moments $\E |T_j(Z)-\E T_j(Z)|^p$ for one or more components, then the expected $p$th power of the Jacobian norm must be correspondingly large.

\subsection{More General Log-Concave Distributions}
A natural question is whether the above results hold only for Gaussian measures. It turns out that similar results also hold for the more general class of log-concave measures $\nu \propto e^{-V}$ for $V$ convex. These include the Laplace and logistic distributions, uniform distributions on convex bodies, exponential distributions, and gamma distributions with shape parameter of at least one. Some of these distributions have been explored as alternative latent priors in DGMs \citep{floto2023tilted}. Definitions of log-concavity and strong log-concavity
are provided in Section~\ref{supp:log-concavity}
of the Supplement.

In the strongly log-concave regime where $V$ satisfies the lower bound $\nabla^2 V \succeq t \cdot I_d$ for $t > 0$, it is known that $C_{\mathrm{PI}}(\nu) \le 1/t$ \citep{klartag2025isoperimetric, bakry2013analysis}, implying Gaussian-like Poincar\'e inequalities. When such a curvature lower bound is not available, it is known that isotropic log-concave measures satisfy the Poincar\'e inequality with constant $C_{\mathrm{PI}}(\nu) \le C \log(ed)$ \citep{bobkov1999logconcave,klartag2023logarithmic}. As such, the dimension dependence is only logarithmic. If a conjecture known as the Kannan-Lov\'asz-Simonovits (KLS) conjecture holds true, then even this $\log(ed)$ factor may be removed \citep{klartag2022bourgain, lee2018kannan}. Hence, our results hold analogously for not just the typical case of Gaussian source distributions, but also for general log-concave measures.

In our setting, these results imply that all moment-to-gradient lower bounds derived from Poincar\'e or $L_p-L_q$-Poincar\'e inequalities extend verbatim to log-concave sources but with constants that are dependent on $\nu$. Concretely, we have the following theorem.
\begin{theorem}[Log-concave Poincar\'e inequalities \citep{bobkov1999logconcave,klartag2023logarithmic}]
\label{thm:log_concave_poincare}
Let $\nu$ be an isotropic log-concave measure on $\mathbb{R}^d$. Then for all smooth $f \in W^{1,2}(\nu)$, we have \begin{align*}
\Var_\nu(f) \le C_{\mathrm{PI}}(\nu)\E_\nu\|\nabla f\|_2^2,
\end{align*}
where $0<C_{\mathrm{PI}}(\nu)\le C\log(ed)$ for a universal constant $C>0$.
\end{theorem}

Similar results also extend to the $L_p-L_q$ case. We state the case for $p = q$ below.
\begin{theorem}[$L_p$-Poincar\'e inequalities for isotropic log-concave measures \citep{milman2009role}]
\label{thm:log_concave_Lp_poincare}
Let $\nu$ be an isotropic log-concave measure on $\mathbb{R}^d$. Then for all smooth $f \in W^{1,p}(\nu)$ where $2 \le p < \infty$, we have
\begin{align*}
\|f-\E_\nu f\|_{L^p(\nu)} \le C_{p,d, \nu} \big\|\|\nabla f\|_{2}\big\|_{L^p(\nu)}.
\end{align*}
where $C_{p,d, \nu} > 0$ is a constant that can depend on $\nu, p$ and $d$.
\end{theorem}

The proof for $L_p-L_q$-Poincar\'e inequalities for log-concave measures is first found in the seminal work of \cite{milman2009role}, and an exposition is in \cite{chewi22}. Thus, our results and analysis for the case of a Gaussian source distribution follow through to the much more general case of a log-concave source. In this scenario, heavy-tailed targets with large moments once again force large Jacobian norms of any pushforward map.

\subsection{Approximate Pushforward Maps}
\label{subsec:tail-stability}

Our previous bounds concern an exact pushforward map. In practice, however, a learned model only approximates the target. It is therefore useful to ask whether the same conclusion continues to hold for an approximation. For the second-moment bound, Wasserstein distance provides a direct answer. For measures $P,Q$ with finite second moments, define
\[
W_2(P,Q)=\inf\big\{(\E\|X-Y\|_2^2)^{1/2}: X\sim P,\ Y\sim Q\big\},
\]
where the infimum is over couplings of the two measures.

\begin{proposition}[Stability of the second-moment bound]
\label{ht:stability}
Let $T$ have components in $W^{1,2}(\gamma)$, let $\widehat\mu=T_\#\gamma$, and suppose that the target $\mu$ has finite second moments. If $W_2(\widehat\mu,\mu)\le\varepsilon$, then
\[
\big(\E_\gamma\|J_T\|_F^2\big)^{1/2}
\ge \big(\sqrt{\tr[\Cov(\mu)]}-\varepsilon\big)_+,
\]
where $(a)_+=\max(a,0)$.
\end{proposition}

Proposition \ref{ht:stability} shows that an approximation in $W_2$ \emph{cannot} eliminate the gradient requirement when the error is small relative to the square root of the target's centered second moment.  The proof of Proposition~\ref{ht:stability} is given in Section~\ref{app:tail-stability} of the Supplement.


\section{Gradient Bounds for Multimodal Targets}
\label{sec:multimodal}

While moment bounds quantify the derivative growth associated with heavy-tailed targets, they do not generally capture the difficulty of transporting to well-separated modes. 
Large moments in the target distribution are not the only source of large Jacobians. 
In this section, we show that accurate representation of target distributions with well-separated modes \emph{also} requires large Jacobians. 

Here, the relevant property is the presence of two regions with substantial probability separated by a low-probability region, which we refer to as a valley. An accurate transport map must send substantial probability mass to each of the two high probability regions while sending little mass to the valley. In what follows, we use isoperimetric inequalities to show that satisfying these requirements forces large Jacobian norms on the transport map, with quantitative bounds determined by the separation and the probability masses of these regions. This generalizes and complements analyses that rely on global Lipschitz assumptions \citep{salmona2022can,tanielian2020}.
All proofs in this section are given in Section~\ref{mm:appendix} of the Supplement.

\subsection{A Localized Jacobian Bound}
\label{mm:localized}

Let $\gamma$ be the standard Gaussian measure on $\R^d$, and let
$T:\R^d\to\R^m$ be locally Lipschitz. We write $P=T_\#\gamma$ for the learned distribution. Local Lipschitz continuity ensures that the
Jacobian $J_T$ exists almost everywhere, but we impose no uniform
bound on its norm. Here, $\|J_T\|_{\mathrm{op}}$ denotes the operator
norm induced by the Euclidean norms.

For a nonempty closed set $A\subseteq\R^m$, define
\begin{equation}
\label{mm:geometry}
A^t=\{y:\operatorname{dist}(y,A)\leq t\},\qquad
U=\{y:0<\operatorname{dist}(y,A)<\Delta\},
\end{equation}
where $\Delta>0$ and the set $U$ separates $A$ from the complement of
$A^\Delta$. We will apply the bound when the target assigns substantial
probability to both $A$ and the complement of
$A^\Delta$ but little probability to $U$.
Let $\varphi$ and $\Phi$ denote the density and cumulative distribution function (CDF) of
a standard univariate Gaussian. Define the Gaussian isoperimetric profile function as
$I(s)=\varphi(\Phi^{-1}(s))$ for $0<s<1$, with
$I(0)=I(1)=0$. Intuitively, the Gaussian isoperimetric profile $I(s)$ quantifies the smallest possible size of the boundary of a region with probability $s$. It expresses the fact that dividing a Gaussian distribution into two regions of substantial probability requires a boundary of non-negligible size. It is a known mathematical fact that this function is symmetric about
$1/2$ and increasing on $[0,1/2]$ \citep{milman2009role}. A general definition and further interpretation of
isoperimetric profiles are given in
Section~\ref{supp:isoperimetric-profile} of the Supplement.

\begin{theorem}[Localized Jacobian bound]
\label{mm:theorem}
Let $T:\R^d\to\R^m$ be locally Lipschitz and $P=T_\#\gamma$.
For every nonempty closed $A\subseteq\R^m$ and every $\Delta>0$,
with $U$ as in \eqref{mm:geometry},
\begin{equation}
\label{mm:master}
\int_{T^{-1}(U)}\|J_T(z)\|_{\mathrm{op}}\,d\gamma(z)
\geq \int_0^\Delta I\big(P(A^t)\big)\,dt
\geq \Delta\cdot I(b),
\end{equation}
where $b=\min\{P(A),1-P(A^\Delta)\}$. \end{theorem}

Note that no integrability assumption on $J_T$ is required; the integral on
the left-hand side may be infinite.
Theorem~\ref{mm:theorem} identifies where the map must have large
Jacobian norms. Intuitively, $I(b)$ gives the minimum size of a boundary needed to partition the source space into two regions each with probability mass at least $b$. We therefore require a contribution of at least
$\Delta \cdot I(b)$ to the Jacobian from the points mapped into the valley $U$.
For fixed $b$, this lower bound increases linearly with the separation $\Delta$.

We stress that our results do not depend on the dimension of the
source or target distributions. The proof of Theorem~\ref{mm:theorem}
 applies Gaussian isoperimetry to
the level sets of the distance from $T(z)$ to $A$ and integrates the
result over distances from $0$ to $\Delta$. Proposition~\ref{mm:logconcave} extends these localized and approximation bounds to general log-concave sources, with a source-dependent constant.

\subsection{Approximation of Multimodal Targets}
\label{mm:approximation}

We next apply the localized Jacobian bound \eqref{mm:master} in
Theorem~\ref{mm:theorem} to maps that approximate a target distribution $\mu$. Define
the total variation distance by
$\operatorname{TV}(P,\mu)=\sup_B|P(B)-\mu(B)|$, where the supremum
is over Borel sets. Suppose that for the sets in
\eqref{mm:geometry},
\begin{equation}
\label{mm:target-geometry}
\mu(A)\geq w_1>0,\qquad
\mu\big((A^\Delta)^c\big)\geq w_2>0,\qquad
\mu(U)\leq\delta.
\end{equation}
These conditions describe the geometry needed for the result without
requiring a density or a mixture representation.

\begin{corollary}[Jacobian moments for a small-probability valley]
\label{mm:corollary}
Suppose that \eqref{mm:target-geometry} holds, and write
$w=\min\{w_1,w_2\}$. Let $T:\R^d\to\R^m$ be locally Lipschitz,
with $P=T_\#\gamma$ and $\operatorname{TV}(P,\mu)\leq\varepsilon<w$.
Then
\begin{equation}
\label{mm:localized-mass}
P(U)\leq\delta+\varepsilon,\qquad
\int_{T^{-1}(U)}\|J_T\|_{\mathrm{op}}\,d\gamma
\geq\Delta I(w-\varepsilon).
\end{equation}
If $\delta+\varepsilon>0$, then, for every $q>1$,
\begin{equation}
\label{mm:moment-bound}
\begin{split}
\big(\E_\gamma\|J_T\|_{\mathrm{op}}^q\big)^{1/q}
&\geq
\left(\int_{T^{-1}(U)}\|J_T\|_{\mathrm{op}}^q\,d\gamma\right)^{1/q}
\\
&\geq
\frac{\Delta I(w-\varepsilon)}
     {(\delta+\varepsilon)^{1-1/q}}.
\end{split}
\end{equation}
\end{corollary}

Bound~\eqref{mm:localized-mass} implies
$\E_\gamma\|J_T\|_{\mathrm{op}}\geq
\Delta I(w-\varepsilon)$. For fixed $w$ and $\varepsilon$, this lower
bound increases linearly with the separation $\Delta$.
Bound~\eqref{mm:moment-bound} additionally quantifies the effect
of a small upper bound $\delta$ on the probability $\mu(U)$
assigned to the separating region $U$. The Jacobian
contribution $\Delta I(w-\varepsilon)$ required by
\eqref{mm:localized-mass} must come from the source points mapped into
$U$, whose total probability is at most $\delta+\varepsilon$.
Concentrating this Jacobian contribution on a set of small probability
forces the higher Jacobian moments to be large. In particular, for each
fixed $q>1$, bound~\eqref{mm:moment-bound} diverges as
$\delta,\varepsilon\to0$ if $\Delta$ and $w$ remain bounded away from zero.

We can obtain similar results for log-concave sources, as formalized in the next proposition. For a probability measure $\nu$ with a log-concave density on $\R^d$, its Cheeger constant $h_\nu>0$ is the largest constant satisfying
\begin{equation}
\label{mm:cheeger-functional}
h_\nu\inf_{c\in\R}\E_\nu|g-c|
\leq\E_\nu\|\nabla g\|_2
\end{equation}
for every bounded locally Lipschitz function $g$ \citep{bobkov1999logconcave,milman2009role}. The Cheeger constant and its connection to the Poincar\'e constant are reviewed in Section~\ref{supp:cheeger-poincare} of the Supplement.

{\hypersetup{linkcolor=blue,filecolor=blue}
\begin{proposition}[Multimodal bounds for log-concave sources]
\label{mm:logconcave}
Let $\nu$ have a log-concave density on $\R^d$, and let $T:\R^d\to\R^m$ be locally Lipschitz, with $P=T_\#\nu$.
For $A$, $\Delta$, and $U$ as in \eqref{mm:geometry}, put
$b=\min\{P(A),1-P(A^\Delta)\}$. Then
\begin{equation}
\label{mm:logconcave-localized}
\int_{T^{-1}(U)}\|J_T\|_{\mathrm{op}}\,d\nu
\geq h_\nu\Delta b.
\end{equation}
If \eqref{mm:target-geometry} holds and
$\operatorname{TV}(P,\mu)\leq\varepsilon<w=\min\{w_1,w_2\}$,
then $P(U)\leq\delta+\varepsilon$. For $\delta+\varepsilon>0$
and every $q>1$,
\begin{equation}
\label{mm:logconcave-moment}
\begin{split}
\big(\E_\nu\|J_T\|_{\mathrm{op}}^q\big)^{1/q}
&\geq
\left(\int_{T^{-1}(U)}\|J_T\|_{\mathrm{op}}^q\,d\nu\right)^{1/q}
\\
&\geq
\frac{h_\nu\Delta(w-\varepsilon)}
     {(\delta+\varepsilon)^{1-1/q}}.
\end{split}
\end{equation}
All integrals are understood in the extended nonnegative sense.
\end{proposition}

Proposition~\ref{mm:logconcave} gives a similar conclusion to
Corollary~\ref{mm:corollary} for log-concave sources. For any fixed
log-concave source $\nu$ and every fixed $q>1$, the lower bound in
\eqref{mm:logconcave-moment} diverges as $\delta,\varepsilon\to0$,
provided $\Delta$ and $w$ remain bounded away from zero.

The strength of the bounds in \eqref{mm:logconcave-localized} and
\eqref{mm:logconcave-moment} depends on the source's Cheeger constant
$h_\nu$. For log-concave $\nu$,
\[
h_\nu\geq c\,C_{\mathrm{PI}}(\nu)^{-1/2},
\]
where $C_{\mathrm{PI}}(\nu)$ is the variance Poincar\'e constant defined
in Section~\ref{sec:gradient-bounds} \citep{milman2009role}.
Thus an upper bound on $C_{\mathrm{PI}}(\nu)$ provides a lower bound on
$h_\nu$, which can be substituted into \eqref{mm:logconcave-localized}
and \eqref{mm:logconcave-moment}.

We now illustrate these bounds using a standard Gaussian source and a
two-component Gaussian mixture target.
}

\begin{example}[A Two-Component Gaussian Mixture]
\label{mm:example}
Fix $a>0$ and consider the Gaussian mixture
\begin{equation}
\label{mm:mixture}
\mu_\sigma=\tfrac12\mathcal N(-a,\sigma^2)
             +\tfrac12\mathcal N(a,\sigma^2),\qquad 0<\sigma\leq1.
\end{equation}
We let $\sigma\downarrow0$ with $a$ fixed; the mixture is bimodal for
$\sigma<a$. Take $A=(-\infty,-a/2]$ and $\Delta=a$, so that
$U=(-a/2,a/2)$.
By symmetry, both side probabilities are
$w_\sigma=(1-\delta_\sigma)/2$, where
\begin{equation}
\label{mm:mixture-valley}
\delta_\sigma=\mu_\sigma(U)
=\Phi\!\left(\frac{3a}{2\sigma}\right)
 -\Phi\!\left(\frac{a}{2\sigma}\right)
\longrightarrow0.
\end{equation}
Consequently, every locally Lipschitz exact transport from a standard
Gaussian source to $\mu_\sigma$ satisfies
\begin{equation}
\label{mm:mixture-moments}
\big(\E_\gamma\|J_T\|_{\mathrm{op}}^q\big)^{1/q}
\geq
\frac{aI((1-\delta_\sigma)/2)}
     {\delta_\sigma^{1-1/q}}
\longrightarrow\infty,\qquad q>1.
\end{equation}
This divergence occurs while every fixed absolute moment of the target
remains bounded.

The same behavior is visible directly in the univariate monotone transport map $T_\sigma=F_\sigma^{-1}\circ\Phi$, where $F_\sigma$ and
$f_\sigma$ are the cumulative distribution function and density of $\mu_\sigma$.
Since $F_\sigma(T_\sigma(z))=\Phi(z)$,
\begin{equation}
\label{mm:quantile-derivative}
T_\sigma'(z)=\frac{\varphi(z)}{f_\sigma(T_\sigma(z))},\qquad
T_\sigma'(0)=\sigma\exp\!\left(\frac{a^2}{2\sigma^2}\right).
\end{equation}
According to \eqref{mm:quantile-derivative}, the derivative at the center diverges as $\sigma\downarrow0$, while the
target density there tends to zero. Equation~\eqref{mm:mixture-moments}
shows that this growth in
Jacobian moments is \emph{intrinsic} to the transport problem, rather than a
consequence of the particular monotone construction.
Details of these calculations are given in Section~\ref{mm:example-proof}
of the Supplement.
\end{example}

\section{Practical Implications for Deep Generative Modeling}
\label{sec:training_implications}

Our results concern properties of \emph{any} pushforward map $T$ that
transports a source distribution $\nu$ to a target $\mu=T_\#\nu$.
Our heavy-tail bounds in Section \ref{sec:gradient-bounds} apply to Gaussian and, more generally, log-concave
sources under weak Sobolev regularity. Meanwhile, our multimodal bounds in Section \ref{sec:multimodal} require only
local Lipschitz continuity. We now outline practical implications of these
results for deep generative models.

\paragraph{Prior Lipschitz-based Analysis}
Prior work studies heavy-tailed targets under global Lipschitz assumptions
on the generative mapping $\widehat T$~\citep{tam2025statistical,wiese2019copula,jaini2020tails}.
Related analyses address disconnected or multimodal targets under global
Lipschitz or bi-Lipschitz assumptions~\citep{tanielian2020,salmona2022can,verine2022}.
A global Lipschitz bound requires
$\|\widehat T(x)-\widehat T(y)\|_2\le L\|x-y\|_2$ for all $x,y$.
This raises the question of whether allowing mappings without such a
bound removes these difficulties.

Our results show that large derivatives can remain necessary without a
global Lipschitz constraint. For example, with $Z\sim N(0,I_d)$,
large component variances of $T(Z)$ require
$\mathbb E\|\nabla T_j(Z)\|_2^2$ to be large for the corresponding
coordinates $j$. This averaged bound does not determine where large
derivatives occur. For the explicit heavy-tail rearrangements in
Section~\ref{sec:examples}, they grow rapidly in the latent tails. For multimodal targets,
our bounds identify a large Jacobian contribution over the preimage of
the low-probability region separating modes, even when target moments
remain bounded. The need for large gradients is therefore intrinsic to
the transport problem.

{\hypersetup{linkcolor=blue}
\paragraph{Implications for Expressivity versus Stability Trade-off}
The requirements of large gradients are related to a practical trade-off between numerical
stability and expressivity. It is known that large gradients can contribute to numerical instability in normalizing flow models \citep{behrmann2021understanding}. For a differentiable mapping $T$, perturbing an
input $z$ by an error $e$ yields
\begin{equation*}
T(z+e)-T(z)=J_T(z)e+o(\|e\|_2),\qquad e\to0,
\end{equation*}
where $J_T(z)$ is the Jacobian at $z$. Large Jacobian norms therefore allow
small input errors in some directions to produce much larger output errors. Such errors can arise from rounding and finite-precision representations in numerical computations. Behrmann et al.~\citep{behrmann2021understanding} document related
numerical failures in reconstruction and sampling, even for
mathematically invertible neural networks. Large gradients are also related to numerical instability in
neural ODE models~\citep{chen2018neural,finlay2020train}, as detailed in the following example.
\begin{example}
In neural ODE models, the transport map $T(z)=X_1(z)$ is obtained by solving the ODE $\dot X_t(z)=v_t(X_t(z))$ over $t\in[0,1]$, with initial condition $X_0(z)=z$. We use $D_zX_t(z)$  to denote the Jacobian of the flow with respect to its initial condition $z$, and $D_xv_t(x)$ to denote the Jacobian of the vector field with respect to its spatial argument $x$.
Assume that $v_t(x)$ and $D_xv_t(x)$ are continuous in
$(t,x)$ and that the flow exists on $[0,1]$. An application of the variational
equation for $D_zX_t(z)$ \citep {teschl2012ordinary} together with Gr\"onwall's inequality yields
\begin{equation}
\label{eq:gronwall_and_variational}
\log\|J_T(z)\|_{\mathrm{op}}
\leq \int_0^1\|D_xv_t(X_t(z))\|_{\mathrm{op}}\,dt.
\end{equation}
where $\|\cdot\|_{\mathrm{op}}$ denotes the operator norm. Large transport
derivatives therefore require a large time integral of the vector-field's
Jacobian norm.  Although the logarithm in \eqref{eq:gronwall_and_variational} moderates this requirement, transport derivatives can grow rapidly as the separation between target modes increases. For the two-component Gaussian mixture in
Proposition~\ref{cm:prop:original-growth}, suppose that the neural
ODE realizes the monotone transport map $T_D$. We then have
\begin{equation}
\label{eq:ode-exponential-bound}
\begin{aligned}
\E\left[
\exp\left(
2\int_0^1\|D_xv_t(X_t(Z))\|_{\mathrm{op}}\,dt
\right)\right]
&\geq \E\big[T_D'(Z)^2\big]\\
&\geq \frac{c}{D}\exp(D^2/8),
\qquad D\geq2,
\end{aligned}
\end{equation}
where $D$ is the distance between the component means, $c>0$, and the last inequality is due to Proposition~\ref{cm:prop:original-growth}. This illustrates how large transport derivatives can be associated with large accumulated vector-field Jacobian norms along individual trajectories.

This is relevant because neural ODEs rely on numerical solvers to
integrate the learned vector field. Finlay et al.~\citep{finlay2020train}
associate large vector-field Jacobian norms with numerical instability in fixed-step solvers and more required function
evaluations in adaptive solvers.
\end{example}

A natural approach to mitigating numerical stability issues
associated with large transport gradients in normalizing flows is
to impose a uniform upper bound on the Jacobian operator norm,
thereby enforcing a global Lipschitz constraint on the transport map~\citep{verine2022}. When such bounds are imposed on both the transport map and
its inverse, the resulting bi-Lipschitz constraints control gradients
in both directions~\citep{verine2022}. While these restrictions can improve numerical stability, they limit the expressivity of the map to
reproduce the target's tails and modal structure. For a Gaussian and log-concaves
sources, global Lipschitz constraints rule out exact representation
of heavy-tailed targets~\citep{jaini2020tails,tam2025statistical}. For multimodal targets, a global Lipschitz bound can prevent accurate
representation of well-separated modes~\citep{salmona2022can}. Verine
et al.~\citep{verine2022} establish related lower bounds on total
variation approximation error for bi-Lipschitz normalizing flows.

Section~\ref{sec:transformations} considers simple, invertible
coordinate transformations when the target density can be evaluated
up to a normalizing constant. Their invertibility allows us to learn
the transport in the transformed coordinates and then recover samples
in the original coordinates by applying the known inverse
transformation. By lightening the tails or bringing separated modes
closer together in the transformed coordinates, these transformations
can reduce the Jacobian norms required of the learned transport map.
}

\section{Approximating Densities via Bijective Reparametrizations}
\label{sec:transformations}

The preceding results describe limitations of learning a pushforward map
in the original coordinates of the target. Several prior works have considered modifying the source distribution to be heavier-tailed to model heavy-tailed targets \citep{jaini2020tails}, including in diffusion models \citep{pandey2024heavy}. Here, we propose a different approach which retains the simple source distribution but reparametrizes the \emph{target} geometry through a change of coordinates. By learning in the transformed coordinates, we aim to reduce the
Jacobian norms required of the transport map. We consider the task of approximating a target density that can be evaluated up to a normalizing constant.

We focus on simple one-dimensional transformations that can
be applied coordinate-wise to multivariate targets. These
transformations simplify marginal geometry, leaving the learned
generative map to model the joint distribution, including
dependence between coordinates.

Let $\Psi:\R^m\to\R^m$ be a bijection such that both $\Psi$ and $\Psi^{-1}$ are $C^1$. We choose $\Psi$ before fitting the learned map and write
$\tau=\Psi_\#\mu$ for the transformed target. We learn a map $S$ from
the source $\nu$ to $\tau$, and generate a sample in the original
coordinates by setting $X=\Psi^{-1}(S(Z))$, where $Z\sim\nu$. The
Jacobian of the resulting map $T=\Psi^{-1}\circ S$ satisfies
\begin{equation}
\label{tf:eq:factorization}
J_T(z)=J_{\Psi^{-1}}(S(z))J_S(z).
\end{equation}
The lower bounds established earlier still apply to $T$. However, the
growth required of $J_T$ can be captured partly by the known inverse
transformation, allowing $S$ to have a smaller expected gradient norm.
We illustrate this approach with a reparametrization for heavy tails and
a reparametrization for multimodal densities. Proofs are given in Section~\ref{tf:app:proofs} of the Supplement.

\subsection{Reparametrizations for Heavy Tails}
\label{tf:subsec:tails}

A simple example of a bijective reparameterization is the inverse hyperbolic sine function,
\begin{equation}
\label{tf:eq:asinh}
\psi_s(x)=\operatorname{arsinh}(x/s),\qquad
\psi_s^{-1}(y)=s\sinh y,\qquad s>0.
\end{equation}
This transformation is approximately linear near the origin and grows
logarithmically in the tails. In particular, both its inverse and its derivative
$\psi_s'(x)=(s^2+x^2)^{-1/2}$  are available in analytical form.

We now show how this reparametrization can lighten heavy tails.
Related definitions of slowly varying functions and regularly varying tails
are given in Section~\ref{supp:regular-variation} of the Supplement.
 
\begin{proposition}[Tail behavior after transformation]
\label{tf:prop:tail-transfer}
Suppose that $\mathbb P(|X|>x)=x^{-\alpha}L(x)$ for all sufficiently
large $x$, where $\alpha>0$ and $L$ is slowly varying. For any fixed
$s>0$, the random variable $Y=\psi_s(X)$ satisfies
\begin{equation}
\label{tf:eq:tail-rate}
\lim_{y\to\infty}\frac{-\log\mathbb P(|Y|>y)}{y}=\alpha.
\end{equation}
Consequently, $\E\exp(t|Y|)<\infty$ for $0<t<\alpha$, and every
polynomial moment of $Y$ is finite.
\end{proposition}

Consequently, in one dimension, the reparametrization \eqref{tf:eq:asinh} converts regularly varying tails into
exponential tails without requiring knowledge of the tail index. We illustrate with an example concerning the Student-$t$ family. 

\begin{example}[Transformed Student-$t$ distributions]
\label{tf:ex:student}
Let $X\sim t_\kappa$, where $\kappa>0$ denotes the degrees of
freedom, and apply the change of coordinates
$Y=\operatorname{arsinh}(X/\sqrt{\kappa})$.
Write $\tau_\kappa$ for the distribution of $Y$.
Substituting $x=\sqrt{\kappa}\sinh y$ into the
change-of-variables formula gives
\begin{equation}
\label{tf:eq:student-density}
f_\kappa(y)
=
\frac{\Gamma((\kappa+1)/2)}
{\sqrt{\pi}\,\Gamma(\kappa/2)}
(\cosh y)^{-\kappa},
\qquad y\in\R.
\end{equation}
Since $\cosh y\sim e^{|y|}/2$ as $|y|\to\infty$, the transformed
density decays exponentially rather than polynomially.
In particular, $Y$ has finite moments of every order.
This illustrates the tail and moment conclusions of
Proposition~\ref{tf:prop:tail-transfer}.

To see the implication for transport gradients, let
$S_\kappa=F_{\tau_\kappa}^{-1}\circ\Phi$, where
$F_{\tau_\kappa}$ and $\Phi$ are the cumulative distribution
functions of $\tau_\kappa$ and the standard Gaussian,
respectively.
Then $S_\kappa(Z)\sim\tau_\kappa$ for $Z\sim N(0,1)$.
For $0<\kappa\le2$, the monotone transport map to the transformed
target satisfies
\begin{equation}
\label{tf:eq:student-energy}
\E\big[(S_\kappa'(Z))^2\big]<\infty,
\end{equation}
as established in Section~\ref{tf:supp:student} of the Supplement.
Thus, changing coordinates yields an exact transport with finite
expected squared derivative.
\end{example}

More generally, the reparametrization~\eqref{tf:eq:asinh}
can be applied coordinate-wise to lighten the marginal tails
of a multivariate target, without requiring independence
between coordinates.

{
\hypersetup{linkcolor=blue,filecolor=blue}
\subsection{Reparametrizations for Multimodal Target Densities}
\label{tf:subsec:component-means}

\leavevmode Approximating multimodal target densities without additional
structural assumptions is challenging, particularly when the
locations of their modes are unknown
\citep{pompe2020framework,salmona2022can}. We therefore focus on
a setting in which reference component locations are available.
These locations may be supplied through prior knowledge of the
model structure or preliminary location estimates. For example,
label permutations in Bayesian mixture models
\citep{stephens2000label} and sign symmetries in Bayesian factor
analysis \citep{erosheva2017reflection} allow an approximate estimate of
one mode to generate reference locations for symmetry-related
modes. We use these reference locations to construct a bijective
reparametrization that brings separated components closer together
in the transformed coordinates, which can reduce the Jacobian
norms required of the learned transport map.

To illustrate this in one dimension, let
$c_1<\cdots<c_K$ be the reference component locations.
Choose $r,h>0$. The parameter $r$ sets the radius of
neighborhoods around these locations whose widths we preserve,
while $h$ controls the compression between them.
Let $d(x)=\min_i|x-c_i|$ denote the distance to the nearest
reference location. We restrict compression to the interval
$[L_\Psi,U_\Psi]$, where $L_\Psi=c_1-r$ and $U_\Psi=c_K+r$.

We define the change of coordinates by
\begin{equation}\label{tf:eq:component-means-map}
 \Psi(x)=\int_0^x a(t)\,dt,
\end{equation}
where its derivative $\Psi'(x)=a(x)$ is given by
\begin{equation}\label{tf:eq:component-means-derivative}
 a(x)=
 \begin{cases}
 \exp\!\left[-\dfrac{(d(x)-r)_+^2}{2h^2}\right],
     & L_\Psi\le x\le U_\Psi,\\
 1,  & x<L_\Psi\text{ or }x>U_\Psi,
 \end{cases}
\end{equation}
and $(u)_+=\max\{u,0\}$. Within each neighborhood
$[c_i-r,c_i+r]$, we have $a(x)=1$, so $\Psi$ simply translates
the interval without changing its width. Between these
neighborhoods, $a(x)<1$, so $\Psi$ compresses the gaps and brings
the components closer together. Smaller $h$ produces stronger
compression.

The derivative \eqref{tf:eq:component-means-derivative} remains strictly positive everywhere and
equals one outside $[L_\Psi,U_\Psi]$, ensuring that $\Psi$ is invertible
on $\R$. The construction of $\Psi$ uses only the reference
component locations and the parameters $r$ and $h$, without requiring
knowledge of the target CDF or component weights. 

We illustrate this reparametrization with a two-component
Gaussian mixture, comparing transport gradients in the original
and transformed coordinates. Consider the model,
\[
\mu_D=\tfrac12N(-D/2,1)+\tfrac12N(D/2,1),
\qquad D>0,
\]
where $D$ is the separation between the component means.
Write $F_D$ and $\Phi$ for the cumulative distribution functions
of $\mu_D$ and the standard Gaussian, respectively.
The monotone transport map from $N(0,1)$ to $\mu_D$ is
$T_D=F_D^{-1}\circ\Phi$.

\begin{proposition}[Derivative bound after reparametrization]
\label{tf:prop:component-locations}
\leavevmode Fix $r>0$ and $0<h<1$. Assume that $0\le\eta\le r$.
Suppose the reference component locations $c_1<c_2$ approximate
the component means with error at most $\eta$:
\[
|c_1+D/2|\le\eta,\qquad |c_2-D/2|\le\eta.
\]
Construct $\Psi$ using
\eqref{tf:eq:component-means-map}--\eqref{tf:eq:component-means-derivative},
and set $S_D=\Psi\circ T_D$. Then $S_D$ transports $N(0,1)$
to the transformed target $\Psi_\#\mu_D$, and
\begin{equation}\label{cm:eq:uniform-bound}
0<S_D'(z)\le C,\qquad
C=2\exp\!\left(\frac{(r+\eta)^2}{2(1-h^2)}\right),\qquad z\in\R.
\end{equation}
Consequently, $\E[S_D'(Z)^2]\le C^2$ for
$Z\sim N(0,1)$, with $C$ independent of $D$.
\end{proposition}

The quantity $r$ reflects the assumed accuracy of the reference component locations, as a smaller assumed error $\eta$ allows a smaller $r$ while maintaining $\eta\le r$. The assumption $\eta\le r$ ensures that each component mean
lies within the neighborhood around its reference location so that the transformed transport
can have uniformly bounded gradients even as the original
component separation increases. We next quantify the gradient
growth of $T_D$ in the original coordinates.

\begin{proposition}[Growth of the original transport]
\label{cm:prop:original-growth}
\leavevmode For the same mixture with $D\ge2$, there is
an absolute constant $c>0$, independent of $D$, such that,
for $Z\sim N(0,1)$,
\begin{equation}\label{cm:eq:original-growth}
\E[T_D'(Z)^2]\ge\frac{c}{D}e^{D^2/8}.
\end{equation}
\end{proposition}

The lower bound in \eqref{cm:eq:original-growth} grows
exponentially with the squared component separation. Increasing separation therefore forces rapid
growth in the expected squared derivative of the original
transport.  Proofs of Propositions \ref{tf:prop:component-locations}--\ref{cm:prop:original-growth}
are given in Section~\ref{supp:component-means} of the Supplement. 

}

\section{Empirical Illustrations}
\label{sec:examples}

We use explicit transport maps from standard Gaussian sources
to examine how target geometry affects transport gradients.
Sections~\ref{subsec:construction} and~\ref{emp:multimodal}
illustrate the gradient growth associated with heavy tails and
separated modes. Section~\ref{emp:transformed-flows} then examines
how the reparametrizations from Section~\ref{sec:transformations}
affect these gradients. Using known transport maps allows us to
study these effects without approximation or optimization error.
We provide additional
empirical illustrations in Section~\ref{app:additional-sims} of the
Supplement.

\subsection{Heavy-Tailed Targets}
\label{subsec:construction}
\label{subsec:sim_setup}

We provide empirical illustrations of our theoretical results through numerical
simulations on commonly used families of heavy-tailed target distributions.
Specifically, we consider the Student-$t$ family with varying degrees of freedom $\kappa$,
as well as the Pareto and Fr\'echet families with different shape parameters $\alpha$ and $\zeta$ respectively. We review related definitions of these canonical families in Section~\ref{appendix_background_tail} of the Supplement.
 As $\kappa, \alpha, \zeta$ decrease, the tails of the corresponding family become heavier.

In one dimension, an explicit transport map between probability measures can be
constructed using quantile functions. Let $\nu$ and $\mu$ be
probability measures on $\mathbb R$ with continuous cumulative distribution
functions $F_\nu$ and $F_\mu$. If $Z\sim\nu$, then the map,
\begin{equation} \label{sim-map}
T(z)=F_\mu^{-1}\!\big(F_\nu(z)\big),
\end{equation}
pushes $\nu$ forward to $\mu$, i.e., $T_\#\nu=\mu$. This monotone transport map is also known as the monotone rearrangement \citep{villani2008optimal} and enjoys optimality properties \citep{santambrogio2015optimal}. 

For each target family, we evaluate the transport map in
\eqref{sim-map} from a standard Gaussian source.
Figure~\ref{fig:jacobian_vertical} compares the numerically evaluated
expected squared Frobenius norm of its Jacobian with the
Poincar\'e lower bound as the target's tail parameter varies.
Figure~\ref{fig:heatmap} depicts the log absolute Jacobian
determinant across the source space.
Section~\ref{supp:heavy-tail-details} of the Supplement
provides further implementation details.

\begin{figure}[!t]

    \begin{subfigure}{0.32\textwidth}

        \includegraphics[width=\textwidth]{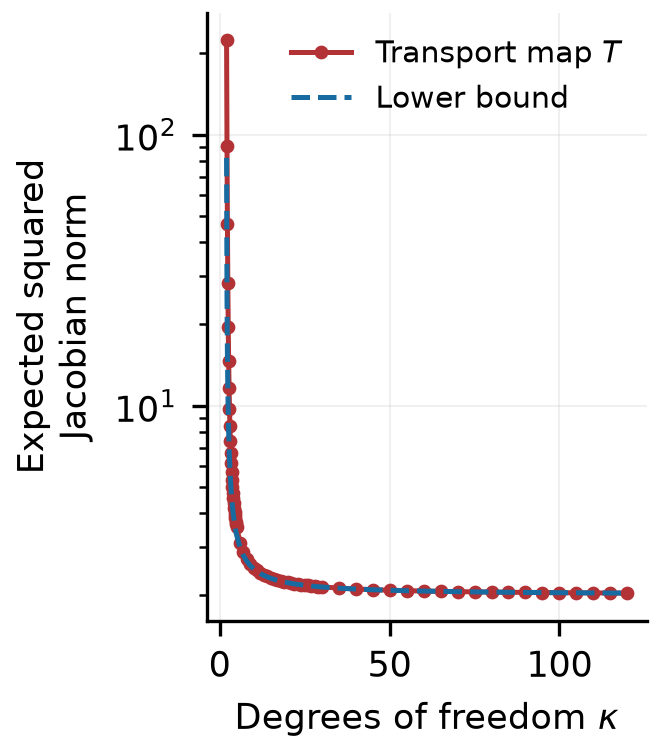}
        \caption{Student-$t$ target}
    \end{subfigure}\hfill
    \begin{subfigure}{0.32\textwidth}
\includegraphics[width=\textwidth]{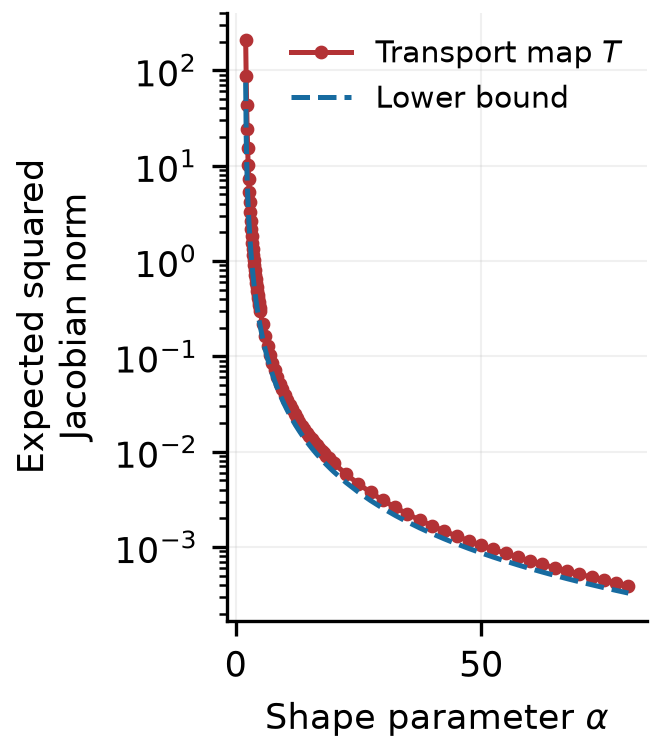}
        \caption{Pareto target}
    \end{subfigure}
\hfill
    \begin{subfigure}{0.32\textwidth}

        \includegraphics[width=\textwidth]{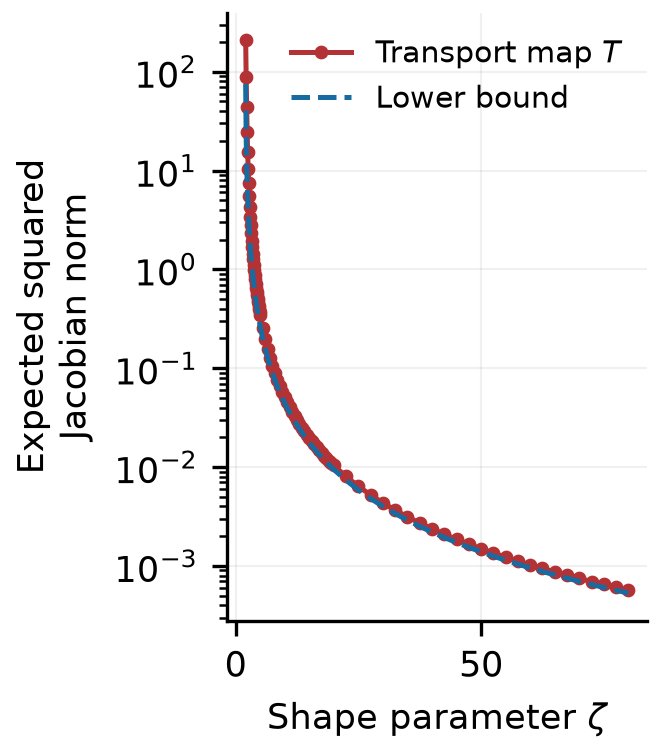}
        \caption{Fr\'echet target}
    \end{subfigure}
    \hfill
    \caption{Comparison of Poincar\'e lower bounds with the expected squared Frobenius norm of the transport Jacobian for different heavy-tailed target distributions under a standard Gaussian source. The red curves with circular markers show the numerically evaluated quantities $\E[\|J_T(Z)\|_F^2]$. The blue dashed curves show the lower bounds.}
    \label{fig:jacobian_vertical}
\end{figure}

\begin{figure*}[!h]
    \centering
    \begin{subfigure}{\textwidth}
        \centering
        \includegraphics[width=.92\textwidth]{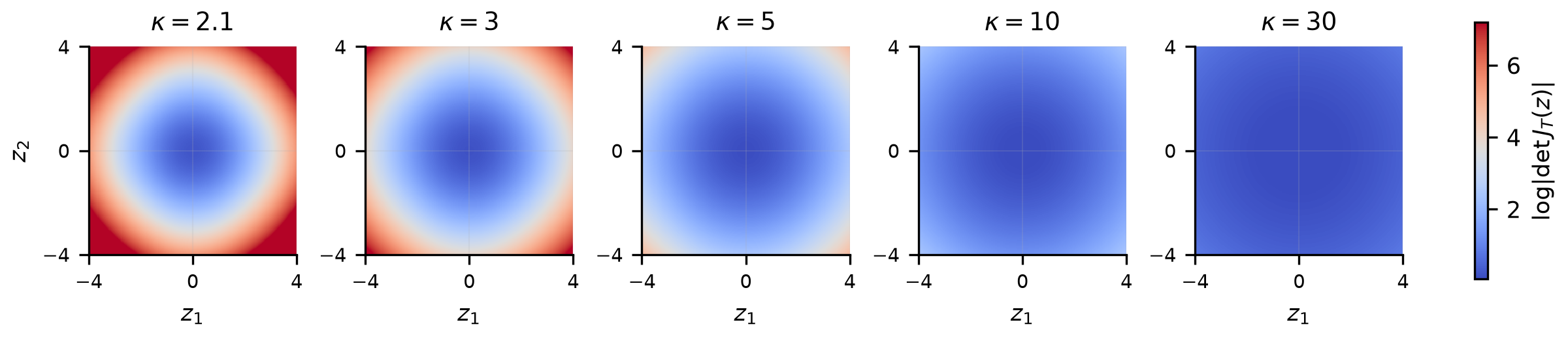}
        \caption{Student-$t$ target}
    \end{subfigure}

    \begin{subfigure}{\textwidth}
        \centering
        \includegraphics[width=\textwidth]{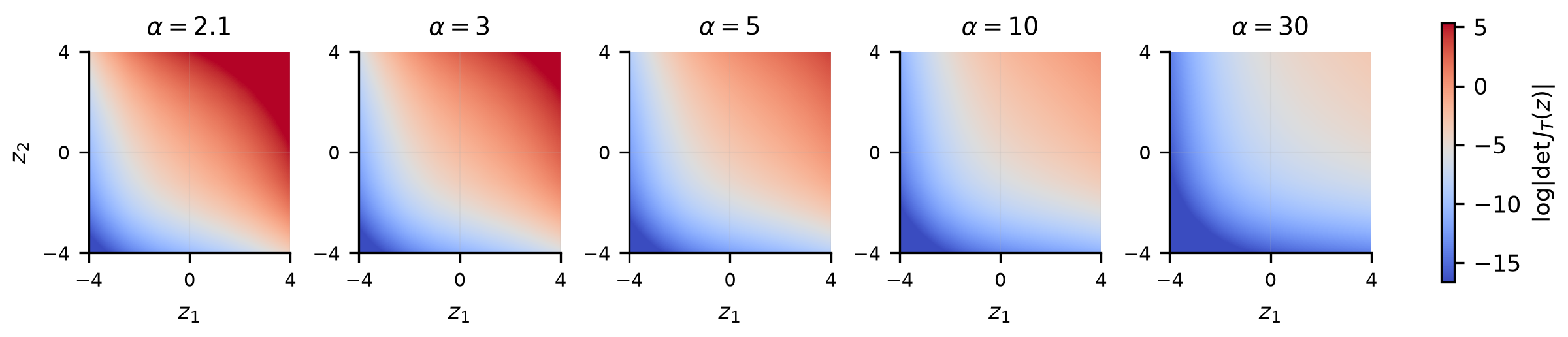}
        \caption{Pareto target}
    \end{subfigure}

    \begin{subfigure}{\textwidth}
        \centering
        \includegraphics[width=\textwidth]{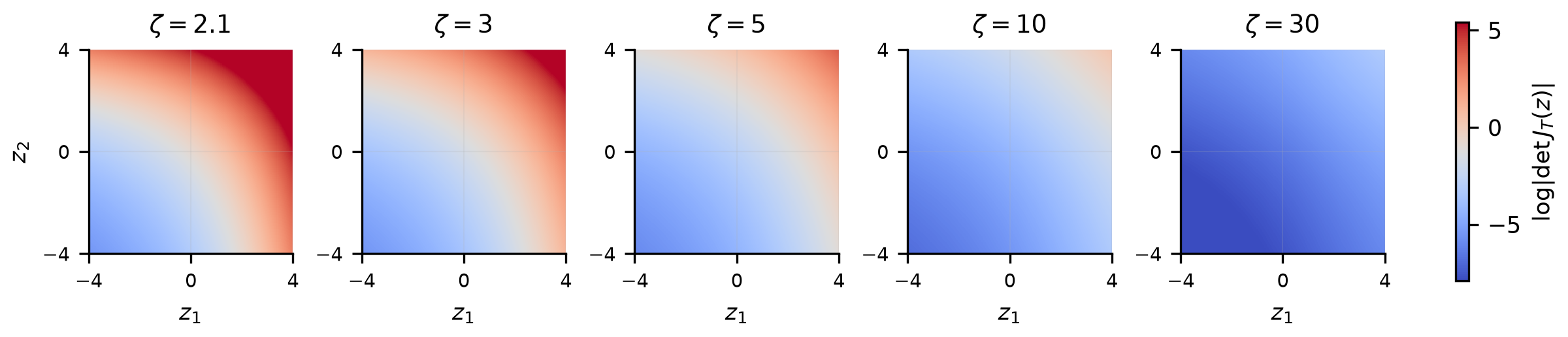}
        \caption{Fr\'echet target}
    \end{subfigure}

    \caption{
    Log absolute Jacobian determinant heatmaps for Gaussian-to-heavy-tailed transport maps. Note that the Pareto and Fr\'echet distributions are right-skewed and supported on the positive orthant.
    }
    \label{fig:heatmap}
\end{figure*}

Figure~\ref{fig:jacobian_vertical} shows that the
Poincar\'e lower bounds from Theorem~\ref{lem:vector_poincare} closely track the numerically evaluated expected
squared Frobenius norms across the three target families.
Both quantities increase sharply as the tail parameters
$\kappa,\alpha,\zeta$ approach $2$ from above, where the corresponding
target variances diverge. These comparisons illustrate how the
moment-based bounds capture the growth of transport gradients
as the targets become heavier-tailed.

Figure~\ref{fig:heatmap} depicts where the transport expands
the source distribution to populate the target's heavy tails.
The log absolute Jacobian determinant increases in these regions
as the tails become heavier.
The Student-$t$ targets have heavy tails in both directions,
whereas the Pareto and Fr\'echet targets are right-skewed
and have a single heavy tail.

\subsection{Multimodal Targets}
\label{emp:multimodal}

We next consider Gaussian mixture targets to illustrate the bounds in
Section~\ref{sec:multimodal}. As in the heavy-tailed examples, the
monotone transport map allows us to examine the transport directly,
without introducing approximation or optimization error. We use the
one-dimensional mixture
$\mu=\tfrac12\mathcal N(-a,\sigma^2)+\tfrac12\mathcal N(a,\sigma^2)$,
where $a=\Delta/2+3\sigma$, and take
$U=(-\Delta/2,\Delta/2)$ as the valley. Increasing $\Delta$ separates
the components while keeping their scales fixed.

Figure~\ref{emp:fig:multimodal} compares the numerically
evaluated mean derivative $\E[T'(Z)]$ with the lower bound in
Theorem~\ref{mm:theorem}. Here, $T$ is the monotone transport
map in \eqref{sim-map} from the standard Gaussian source to
the mixture target. The expectation averages its derivative
over source points $Z\sim N(0,1)$. Both the lower bound and
the computed mean derivative increase with the width of the
low-probability interval. We observe that the lower bound
captures the rate of this increase well. 

\begin{figure}[t]
\centering
\includegraphics[width=0.98\textwidth]{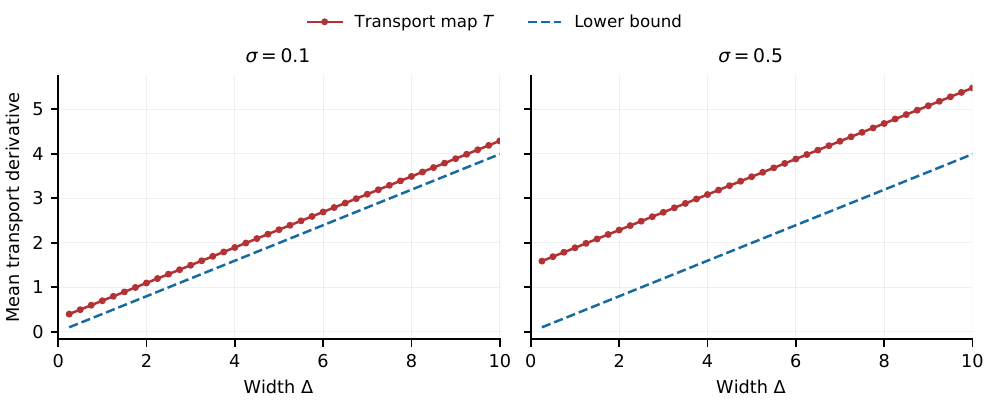}
\caption{Gradient bounds for two-component Gaussian mixtures.
The plots compare $\E[T'(Z)]$ with the lower bound
in Theorem~\ref{mm:theorem} as the width $\Delta$ of the low-probability interval
increases.
The two panels use component standard deviations
$\sigma=0.1$ and $\sigma=0.5$.}
\label{emp:fig:multimodal}
\end{figure}

\subsection{Reparametrizations of Known Transport Maps}
\label{emp:transformed-flows}
{\hypersetup{linkcolor=blue,filecolor=blue}

We next examine how changing the target coordinates affects
transport derivatives. For each scenario below, we consider  a known map $T$ that transports a standard
Gaussian source to the original target, so $X=T(Z)$ is a
sample from that target. Applying the fixed coordinate
transformation $\Psi$ gives $Y=\Psi(X)$. The composition
$S=\Psi\circ T$ therefore transports the same Gaussian source
to the target expressed in transformed coordinates.
We compare $\E[\|J_T(Z)\|_F^2]$ and
$\E[\|J_S(Z)\|_F^2]$, which measure the squared Jacobian norms
averaged over the Gaussian source. In one dimension, these
quantities are $\E[T'(Z)^2]$ and $\E[S'(Z)^2]$, respectively.

For the heavy-tailed examples, we use the three canonical
families from Section~\ref{subsec:construction}, namely the Student-$t$,
Pareto, and Fr\'echet distributions.
Each target is a two-dimensional product distribution with
independent coordinates, and the source is $N(0,I_2)$.
We apply the asinh transformation in \eqref{tf:eq:asinh}
coordinate-wise, using $s=\sqrt{\kappa}$ for the Student-$t$
family and $s=1$ for the Pareto and Fr\'echet families.
We vary each family's tail parameter from $1$ to $10$.

For the multimodal examples, we use the one-dimensional
Gaussian mixture
\[
\mu_D=\frac15\sum_{i=1}^5 N(m_i,1),
\qquad m_i=D(i-3),
\qquad D\in\{3,4,\ldots,16\}.
\]
Thus, $D$ controls the separation between adjacent component
means while the component weights and variances remain fixed.
We construct $\Psi$ using
\eqref{tf:eq:component-means-map}--\eqref{tf:eq:component-means-derivative},
with $r=1.5$ and $h=0.4$.
We specify approximate reference component locations $c_i$
satisfying $|c_i-m_i|\le0.5$. The differences $c_i-m_i$ are
listed in Section~\ref{cm:sec:numerics} of the Supplement,
which also describes the numerical integration procedure.

\begin{figure}[t]
\centering
\captionsetup{labelfont={color=black},textfont={color=black}}
\includegraphics[width=0.98\textwidth]{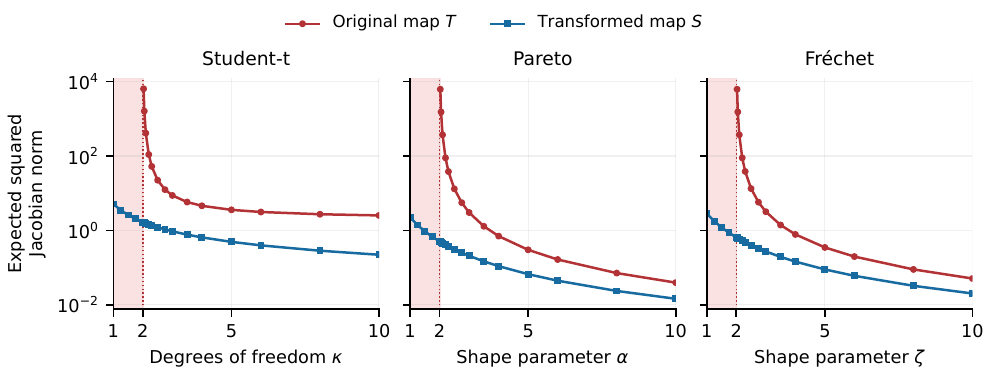}
\caption{Expected squared Frobenius Jacobian norms before and after the asinh reparametrization for two-dimensional product Student-$t$, Pareto, and Fr\'echet targets, using a standard Gaussian source. The shaded regions indicate degrees of freedom or shape parameters less than or equal to $2$.}
\label{emp:fig:transformed-tails}
\end{figure}

Figure~\ref{emp:fig:transformed-tails} shows that the
asinh reparametrization substantially reduces the expected
squared Jacobian norms across all three target families.
The original-map values increase sharply as the tail parameter
approaches $2$ from above.
The transformed-map values remain moderate across the displayed
range. These results are in line with Proposition~\ref{tf:prop:tail-transfer}, which explains how the
transformation converts heavy tails into exponential
tails.  

\begin{figure}[t]
\centering
\captionsetup{labelfont={color=black},textfont={color=black}}
\includegraphics[width=0.72\textwidth]{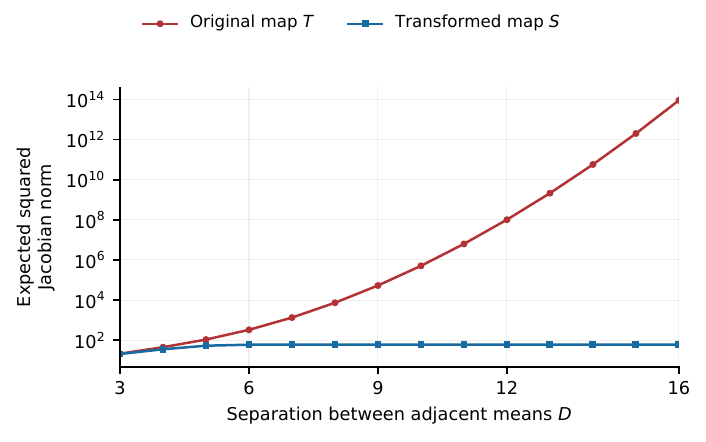}
\caption{Expected squared derivatives $\E[T'(Z)^2]$ and $\E[S'(Z)^2]$, for $Z\sim N(0,1)$, before and after transformation for five-component Gaussian mixtures as adjacent means move farther apart. Component widths, transformation parameters, and errors in the reference component locations are held fixed.}
\label{emp:fig:transformed-modes}
\end{figure}

Figure~\ref{emp:fig:transformed-modes} shows that the
original transport's expected squared derivative increases
rapidly as adjacent component means move farther apart.
After reparametrization, the corresponding quantity levels off,
although the component separation continues to increase.
This five-component example exhibits the same qualitative
contrast established for two-component mixtures by
Propositions~\ref{tf:prop:component-locations}
and~\ref{cm:prop:original-growth}: the transformed transport
admits a derivative bound independent of separation, whereas
the original transport's expected squared derivative grows
exponentially with squared separation.

}

%
%

\section*{Declaration of AI Use}
\noindent The authors developed all core logic and theoretical frameworks in this paper and drafted the initial manuscript. AI tools (ChatGPT 6 Astra and Claude Opus 5.5) were used to refine the writing, check mathematical derivations, and assist in code writing. After using these tools, the authors carefully reviewed and edited the content as necessary. The authors take full responsibility for the final manuscript.

\bibliographystyle{imsart-number-ed}
\bibliography{references}

\clearpage
\markboth{SUPPLEMENTARY MATERIAL}{SUPPLEMENTARY MATERIAL}
\section*{Supplement to ``Intrinsic Gradient Barriers in Pushforward Deep Generative Models''}
\setcounter{section}{0}
\setcounter{figure}{0}
\setcounter{table}{0}
\renewcommand{\thesection}{S\arabic{section}}
\renewcommand{\thefigure}{S\arabic{figure}}
\renewcommand{\thetable}{S\arabic{table}}
\renewcommand{\theHsection}{S\arabic{section}}
\renewcommand{\theHfigure}{S\arabic{figure}}
\renewcommand{\theHtable}{S\arabic{table}}
\captionsetup{labelfont={color=black},textfont={color=black}}
{
\section{Technical Preliminaries}
\label{append_technical}
\subsection{Weak Gradients}

Let $\nu$ be a full-dimensional log-concave probability measure on $\mathbb{R}^d$, and set $\Omega=\operatorname{int}(\operatorname{supp}\nu)$. A function $f \in L^p(\nu)$ that is locally integrable on $\Omega$ is said to admit a \emph{weak gradient} if there exists a vector-valued function $g=(g_1,\dots,g_d) \in L^p(\nu;\mathbb{R}^d)$, also locally integrable on $\Omega$, such that
\begin{align*}
    \int_{\Omega} f(x)\, \partial_i \varphi(x)\, dx
    \;=\;
    - \int_{\Omega} g_i(x)\, \varphi(x)\, dx,
\end{align*}
for all smooth, compactly supported test functions $\varphi \in C_c^\infty(\Omega)$ and all $i=1,\dots,d$.
In this case, we write $\nabla f = g$ and refer to $\nabla f$ as the weak gradient of $f$.

\paragraph{Examples.}
For functions that are classically differentiable, the weak gradient coincides with the classical gradient, so the definition introduces nothing new for smooth maps. The notion becomes useful for functions with isolated non-smoothness:
\begin{itemize}
\item The function $f(x) = |x|$ on $\mathbb{R}$ is not differentiable at $0$, but admits a weak derivative $f'(x) = \mathrm{sgn}(x)$, defined almost everywhere.
\item The ReLU function $\mathrm{ReLU}(x) = \max(x,0)$ admits the weak derivative $\mathbf{1}\{x > 0\}$.
\end{itemize}
Generally, on an intuitive level, piecewise-smooth functions with finitely many ``kinks'' admit weak gradients given by the classical gradient on each smooth piece. This level of generality is the natural one in our setting for three reasons: (i) many neural-network-based maps are only piecewise smooth; (ii) weak derivatives are the standard notion of derivative in Sobolev spaces and in the formulation of Poincar\'e-type inequalities; and (iii) they are the natural notion of derivative used in PDE-based analyses of flow- and transport-based generative models. We refer the reader to standard texts on Sobolev spaces and weak derivatives, such as \cite{evans2022partial}, for further background.

\subsection{Sobolev Spaces}

Any rigorous analysis of pushforward maps requires minimal regularity assumptions to ensure that gradients and expectations thereof are well defined. Since our main analytical tools are Poincar\'e inequalities, which relate moments of a distribution to integrals involving gradients, we work with weak notions of differentiability formalized through Sobolev spaces.

For $1\le p<\infty$, $L^p(\nu)$ is the space of measurable functions $f$ satisfying $\int |f|^p\,d\nu<\infty$, with norm
\[
\|f\|_{L^p(\nu)}
=\left(\int |f(x)|^p\,d\nu(x)\right)^{1/p}.
\]
For $p=\infty$, the norm is the essential supremum with respect to $\nu$.

For $1 \le p < \infty$, the Sobolev space $W^{1,p}(\nu)$ consists of all functions $f \in L^p(\nu)$ that admit a weak gradient $\nabla f \in L^p(\nu;\mathbb{R}^d)$. The associated norm is defined by
$$
\|f\|_{W^{1,p}(\nu)}^p
\;=\;
\|f\|_{L^p(\nu)}^p + \|\nabla f\|_{L^p(\nu)}^p.
$$

In particular, the space $W^{1,2}(\nu)$ plays a central role in our analysis, as it provides the natural domain for Poincar\'e and log-Sobolev inequalities. Membership in $W^{1,2}(\nu)$ requires only weak differentiability and square-integrability of the gradient, imposing substantially weaker regularity than global Lipschitz conditions. We refer to \cite{evans2022partial} and \cite{brezis2011functional} for comprehensive treatments of Sobolev spaces and weak derivatives, and \cite{bogachev1998gaussian} for details on Gaussian Sobolev spaces.
\subsection{Background on Heavy-Tailed Distributions}
\label{subsec:heavytails}

A central object in this work is the tail behavior of the target distribution. Tail decay is closely tied to the existence and growth of moments. For a real-valued random variable $X$, we use the two-sided convention: $X$ is heavy-tailed if $\E[e^{t|X|}]=\infty$ for every $t>0$. In practice, heavy-tailed distributions typically possess only finitely many moments, whereas light-tailed distributions admit moments of all orders. We briefly review relationships between heavy tails, $p$-th moments, and polynomial tail bounds; see \cite{roch2024modern} for additional background.

\subsubsection{Moments and Polynomial Tail Bounds}

Let $X \sim \nu$. For $p \ge 1$, the $p$-th moment of $\nu$ is
$$
\mathbb{E}_\nu \|X\|^p \;=\; \int_{\mathbb{R}^d} \|x\|^p \, d\nu(x),
$$
whenever the integral is finite. A basic consequence of Markov's inequality is the polynomial tail bound
\begin{equation}
\label{eq:markov_poly_tail}
\mathbb{P}\big(\|X\| \ge t\big)
\;\le\;
\frac{\mathbb{E}\|X\|^p}{t^p},
\qquad t>0,
\end{equation}
which shows that finiteness of the $p$-th moment implies at least $t^{-p}$ control of the tail probability.

Conversely, moments can be expressed in terms of the tail function via the standard identity
\begin{equation}
\label{eq:moment_tail_identity}
\mathbb{E}\|X\|^p
\;=\;
p \int_{0}^{\infty} t^{p-1} \, \mathbb{P}\big(\|X\| \ge t\big)\, dt,
\end{equation}
whenever the integral is finite. In particular, if $\mathbb{P}(\|X\|\ge t) \lesssim t^{-\alpha}$ for some $\alpha>0$, then (\ref{eq:moment_tail_identity}) implies $\mathbb{E}\|X\|^p < \infty$ for all $p<\alpha$. Conversely, if $\mathbb{P}(\|X\|\ge t) \gtrsim t^{-\alpha}$ for large $t$, then moments of order $p\ge \alpha$ diverge. Thus, polynomial tail decay corresponds to the existence of moments only up to a critical order.

Complementary lower bounds on tail probabilities can be obtained under additional assumptions via the Paley--Zygmund inequality. Specifically, if $Y \ge 0$ has finite second moment, then for any $0 < s < 1$,
\begin{align*}
\mathbb{P}\big( Y \ge s \, \mathbb{E}Y \big)
\;\ge\;
(1-s)^2 \, \frac{(\mathbb{E}Y)^2}{\mathbb{E}Y^2}.
\end{align*}
Applying this inequality with $Y = \|X\|^p$ yields
\begin{align*}
\mathbb{P}\Big( \|X\| \ge s^{1/p} (\mathbb{E}\|X\|^p)^{1/p} \Big)
\;\ge\;
(1-s)^2 \, \frac{(\mathbb{E}\|X\|^p)^2}{\mathbb{E}\|X\|^{2p}},
\end{align*}
which shows that, provided $\mathbb{E}\|X\|^{2p}$ is controlled relative to $(\mathbb{E}\|X\|^p)^2$, large $p$-th moments necessarily imply non-negligible mass in the tails. We refer to \cite{dufour2015moments} for further discussion.

\subsubsection{Slow and Regular Variation}
\label{supp:regular-variation}

We recall the following standard definitions; see
Bingham et al.~\citep{bingham1987regular} for further background.

A positive measurable function $L$ is \emph{slowly varying at infinity} if
\[
\lim_{x\to\infty}\frac{L(cx)}{L(x)}=1
\qquad\text{for every fixed }c>0.
\]
 
A positive measurable function $g$ is \emph{regularly varying at
infinity with index $\rho\in\R$} if
\[
\lim_{x\to\infty}\frac{g(cx)}{g(x)}=c^\rho
\qquad\text{for every fixed }c>0.
\]
Equivalently, $g(x)=x^\rho L(x)$ for a slowly varying function $L$.
Slow variation corresponds therefore to the special case $\rho=0$.

We say that a real-valued random variable $X$ has
\emph{regularly varying tails with tail index $\alpha>0$} if
its tail function $x\mapsto\mathbb P(|X|>x)$ is regularly varying
with index $-\alpha$. Equivalently,
\[
\mathbb P(|X|>x)=x^{-\alpha}L(x)
\]
for all sufficiently large $x$, where $L$ is slowly varying.
This describes power-law tail decay with a slowly varying
multiplicative factor. For example, a Student-$t$ distribution
with $\kappa$ degrees of freedom has tail index $\kappa$.
\subsubsection{Some Canonical Heavy-Tailed Distributions} \label{appendix_background_tail}

In our simulations, we consider three classical families of distributions that exhibit heavy-tailed behavior. We give more details on these families below.
The Student-$t$ distribution is a classical family of heavy-tailed distributions
that interpolates between Gaussian and power-law behavior. A univariate Student-$t$ random variable with $\kappa>0$ degrees of freedom has density,
\begin{equation} \label{eq:pdf-student-t}
f(z)
=
\frac{\Gamma\!\left(\frac{\kappa+1}{2}\right)}
{\sqrt{\kappa\pi}\,\Gamma\!\left(\frac{\kappa}{2}\right)}
\left(1+\frac{z^2}{\kappa}\right)^{-(\kappa+1)/2}.
\end{equation}
The degrees of freedom parameter $\kappa$ controls the tail behavior of the Student-$t$ distribution. As $\kappa \to \infty$, the Student-$t$ converges to a standard Gaussian. As $\kappa$ becomes smaller, it exhibits increasingly heavy-tailed behavior.

Unlike the Student-$t$ family, the Pareto distribution exhibits an exact polynomially decaying right tail. A univariate Pareto random variable with shape parameter
$\alpha>0$ and scale parameter $\lambda > 0$ has density,
\begin{equation} \label{eq:pareto-pdf}
    f(z)=\alpha\,\lambda^{\alpha}\,z^{-(\alpha+1)}, \; z\ge \lambda.
\end{equation}
The Fr\'echet distribution is another classical heavy-tailed distribution with a polynomially decaying right tail. A
univariate Fr\'echet random variable with shape parameter $\zeta>0$ and scale
parameter $\lambda>0$ has density,
\begin{equation} \label{eq:frechet-pdf}
f(z)=\zeta\,\lambda^{\zeta}\,z^{-(\zeta+1)}
\exp\!\left(-\left(\frac{\lambda}{z}\right)^{\zeta}\right), \;z>0.
\end{equation}
In both the Pareto and Fr\'echet families, the respective shape parameters $\alpha$ and $\zeta$ directly control tail heaviness and moment existence, in the sense that
the $k$th moment exists if and only if $k<\alpha$ and $k < \zeta$ respectively. The smaller $\alpha$ and $\zeta$ are, the heavier the upper tails. In particular, the variances for the Pareto and Fr\'echet are finite only for $\alpha, \zeta>2$, and they diverge as $\alpha, \zeta \downarrow 2$.

\subsection{Log-Concave and Strongly Log-Concave Distributions}
\label{supp:log-concavity}

A distribution with density $p$ on $\R^d$ is
\emph{log-concave} if $\log p$ is concave, with the convention
$\log 0=-\infty$. Equivalently, $p(x)\propto e^{-V(x)}$
for a convex function $V$, allowing $V=+\infty$ where
$p=0$. In terms of the density, this means that
\[
p((1-\theta)x+\theta y)
\ge p(x)^{1-\theta}p(y)^\theta,
\qquad x,y\in\R^d,\quad 0<\theta<1.
\]
Thus, the density cannot dip below the geometric interpolation
of its values at two points. Gaussian and Laplace distributions,
and uniform distributions on convex bodies, are examples.

The distribution is \emph{$t$-strongly log-concave}, for
$t>0$, if $V(x)-t\|x\|_2^2/2$ is convex. When
$V\in C^2(\R^d)$, this is equivalent to
\[
\nabla^2V(x)\succeq tI_d,\qquad x\in\R^d.
\]
This condition places a positive lower bound on the curvature
of $V$ in every unit direction. The standard Gaussian is
$1$-strongly log-concave.
See Saumard and Wellner~\citep{saumard2014logconcavity}
for a detailed review, and Chewi~\citep{chewi22} and
Bakry et al.~\citep{bakry2013analysis} for related functional
inequalities.

\subsection{Isoperimetric Profiles}
\label{supp:isoperimetric-profile}

An isoperimetric profile describes the smallest boundary
measure of a region having a specified probability.
For a Borel set $A\subseteq\R^d$, let
$A^\varepsilon=\{x:\operatorname{dist}(x,A)\le\varepsilon\}$.
Its boundary measure under $\nu$ is
\[
\nu^+(A)=
\liminf_{\varepsilon\downarrow0}
\frac{\nu(A^\varepsilon)-\nu(A)}{\varepsilon}.
\]
This measures how quickly probability is gained when the set
is slightly enlarged. The isoperimetric profile of $\nu$ is
\[
I_\nu(s)=
\inf_{\substack{A\subseteq\R^d\ \mathrm{Borel}\\\nu(A)=s}}
\nu^+(A),\qquad 0<s<1.
\]
Thus, every region of probability $s$ has boundary measure
at least $I_\nu(s)$.

For the standard Gaussian measure $\gamma$ in any dimension,
\[
I_\gamma(s)=I(s)=\varphi(\Phi^{-1}(s)),\qquad 0<s<1,
\]
where $\varphi$ and $\Phi$ are the standard univariate Gaussian
density and CDF. Half-spaces attain this bound.
This gives the interpretation used in
Section~\ref{sec:multimodal}: separating regions of prescribed
Gaussian probability requires a boundary of at least this size.
See Ledoux~\citep{ledoux2001concentration},
Milman~\citep{milman2009role}, and Chewi~\citep{chewi22}
for further background.

\subsection{Cheeger and Poincar\'e Constants}
\label{supp:cheeger-poincare}

Using the boundary measure defined above, the
\emph{Cheeger constant} of $\nu$ is
\[
h_\nu=
\inf_{\substack{A\subseteq\R^d\ \mathrm{Borel}\\0<\nu(A)<1}}
\frac{\nu^+(A)}{\min\{\nu(A),1-\nu(A)\}}.
\]
A larger $h_\nu$ means that separating two regions with
substantial probability requires a larger boundary measure.
Equivalently, $h_\nu$ is the largest constant satisfying
\[
h_\nu\inf_{c\in\R}\E_\nu|g-c|
\le \E_\nu\|\nabla g\|_2
\]
for every bounded locally Lipschitz function $g$.
This is the form used in Proposition~\ref{mm:logconcave}.

The \emph{variance Poincar\'e constant}
$C_{\mathrm{PI}}(\nu)$ is the smallest constant satisfying
\[
\operatorname{Var}_\nu(f)
\le C_{\mathrm{PI}}(\nu)\,
\E_\nu\|\nabla f\|_2^2,
\qquad f\in W^{1,2}(\nu).
\]
Smaller values give stronger control of a function's variance
through its expected squared gradient.
For the standard Gaussian measure,
$C_{\mathrm{PI}}(\gamma)=1$.

For log-concave $\nu$, the two constants are related by
\[
h_\nu\ge c\,C_{\mathrm{PI}}(\nu)^{-1/2}
\]
for a universal constant $c>0$.
This connects the source-dependent constants in our
moment-based and multimodal gradient bounds.
See Ledoux~\citep{ledoux2001concentration},
Bakry et al.~\citep{bakry2013analysis}, and
Milman~\citep{milman2009role} for further background.

\section{Proofs for Section~\ref{sec:gradient-bounds}} \label{append_proofs}
\label{app:tail-stability}

\label{app:poincare_proof}

\begin{proof}[Proof of Lemma \ref{thm:gaussian_poincare}]
This result and proof can be found in Theorem 1.6.4 of \cite{bogachev1998gaussian}.
\end{proof}

\begin{proof}[Proof of Theorem~\ref{lem:vector_poincare}]
(i) Apply the scalar Gaussian Poincar\'e inequality to each component
$T_j \in W^{1,2}(\gamma)$.

(ii) Fix $a\in\mathbb{R}^m$ and consider $g_a(z)=a^\top T(z)$.
Since $g_a\in W^{1,2}(\gamma)$ and $\nabla g_a(z)=J_T(z)^\top a$,
the scalar Poincar\'e inequality yields
\begin{align*}
\Var(a^\top T(Z)) &\le \E\big[\|J_T(Z)^\top a\|_2^2\big] \\
&= a^\top \E\!\big[J_T(Z)J_T(Z)^\top\big] a.
\end{align*}
Using $\Var(a^\top T(Z)) = a^\top \Cov(T(Z)) a$ for all $a$ gives
$\Cov(T(Z)) \preceq \E[J_T(Z)J_T(Z)^\top]$.

(iii) Taking traces of the PSD inequality in (ii) gives
\begin{align*}
\tr[\Cov(T(Z))] &\le \tr\!\big(\E[J_T(Z)J_T(Z)^\top]\big)
\\ &= \E\,\tr(J_T(Z)J_T(Z)^\top) =\E\|J_T(Z)\|_F^2.
\end{align*}

(iv) Let $f=T_i$ and $g=T_j$. By Cauchy--Schwarz,
\begin{align*}
&|\Cov(f(Z),g(Z))| \\
& \qquad = \big|\E[(f(Z)-\E f(Z))(g(Z)-\E g(Z))]\big| \\
& \qquad \le \sqrt{\Var(f(Z))}\,\sqrt{\Var(g(Z))}.
\end{align*}
Applying the scalar Gaussian Poincar\'e inequality to $f$ and $g$ gives
$\Var(f(Z))\le \E\|\nabla f(Z)\|_2^2$ and $\Var(g(Z))\le \E\|\nabla g(Z)\|_2^2$,
hence
$$
|\Cov(f(Z),g(Z))|
\le
\sqrt{\E\|\nabla f(Z)\|_2^2}\;
\sqrt{\E\|\nabla g(Z)\|_2^2}.
$$

\end{proof}

\begin{proof}[Proof of Lemma \ref{thm:gaussian_lp_poincare}]
   This result and proof can be found in Corollary 1.7.3 of \cite{bogachev1998gaussian}.
\end{proof}

\begin{proof}[Proof of Corollary \ref{cor:vector_gaussian_lp_poincare}]
Apply Lemma~\ref{thm:gaussian_lp_poincare} to each component $T_j$:
$$
\|T_j-\E_\gamma T_j\|_{L^p(\gamma)} \le C_p\,\|\|\nabla T_j\|_2\|_{L^p(\gamma)}.
$$
Raising to the power $p$ and summing over $j$ gives
$$
\sum_{j=1}^m \|T_j-\E_\gamma T_j\|_{L^p(\gamma)}^p
\le
C_p^p \sum_{j=1}^m \E_\gamma \|\nabla T_j(Z)\|_2^p.
$$
Next, since $\|J_T(z)\|_F^2=\sum_{j=1}^m \|\nabla T_j(z)\|_2^2$, for $p\ge2$ we have pointwise
\begin{align*}
\sum_{j=1}^m \|\nabla T_j(z)\|_2^p
&=
\sum_{j=1}^m \big(\|\nabla T_j(z)\|_2^2\big)^{p/2} \\
&\le
\Big(\sum_{j=1}^m \|\nabla T_j(z)\|_2^2\Big)^{p/2} \\
&=
\|J_T(z)\|_F^p.
\end{align*}
Taking expectations yields $\sum_{j=1}^m \E_\gamma \|\nabla T_j(Z)\|_2^p \le \E_\gamma \|J_T(Z)\|_F^p$.
This completes the proof with $C'_p := C_p^p$.
\end{proof}
 
\begin{proof}[Proof of Theorem~\ref{thm:log_concave_poincare}]
Suppose first that $d=1$. Since $\nu$ is isotropic,
$\Var_\nu(x)=1$. Corollary~4.3 of
\cite{bobkov1999logconcave} gives
\[
C_{\mathrm{PI}}(\nu)
\le 12\Var_\nu(x)=12.
\]
Here we use that the spectral gap in that reference is
the reciprocal of our variance Poincar\'e constant.

For $d\ge2$, Theorem~1.2 together with display~(1.4) of
\cite{klartag2023logarithmic} yields
\[
C_{\mathrm{PI}}(\nu)\le C_0\log d
\]
for a universal constant $C_0>0$.
Consequently, in every dimension $d\ge1$,
\[
C_{\mathrm{PI}}(\nu)\le C\log(ed),
\qquad C=\max\{12,C_0\}.
\]
The corresponding Poincar\'e inequalities apply to the
smooth functions $f\in W^{1,2}(\nu)$ in the theorem,
which proves the claim.
\end{proof}

\begin{proof}[Proof of Theorem \ref{thm:log_concave_Lp_poincare}]
  Theorem~\ref{thm:log_concave_poincare} establishes the usual
Poincar\'e inequality for $\nu$. By Theorem 2.4 and
Lemma 2.1 of \cite{milman2009role}, for every $2\le p<\infty$,
\[
\|f-\operatorname{med}_\nu(f)\|_{L^p(\nu)}
\le \widetilde C_{p,d,\nu}\,
\big\|\|\nabla f\|_2\big\|_{L^p(\nu)},
\]
where $\operatorname{med}_\nu(f)$ is a median of $f$ under $\nu$, and
$\widetilde C_{p,d,\nu}$ depends on $p$ and the Poincar\'e constant of $\nu$. Moreover,
\[
\|f-\E_\nu f\|_{L^p(\nu)}
\le \|f-\operatorname{med}_\nu(f)\|_{L^p(\nu)}
   +|\E_\nu f-\operatorname{med}_\nu(f)|
\le 2\|f-\operatorname{med}_\nu(f)\|_{L^p(\nu)}.
\]
Combining these inequalities proves the result with
$C_{p,d,\nu}=2\widetilde C_{p,d,\nu}$.
\end{proof}

\begin{remark}
The Poincar\'e inequalities above extend from smooth functions to their respective Sobolev classes with the same constants. These classes include locally Lipschitz functions with the required function and gradient integrability.

Proposition~27 of Barthe and Klartag~\citep{bartheklartag2020spectral} establishes that restrictions of functions in $C_c^\infty(\mathbb R^d)$ to $\Omega$ are dense in $W^{1,2}(\nu)$, with weak derivatives taken on $\Omega$. Their proof applies equally with $L^p$ norms replacing $L^2$ norms, giving density in $W^{1,p}(\nu)$ for every $1\le p<\infty$.

Thus, for $f\in W^{1,p}(\nu)$, choose smooth $f_n\to f$ in the Sobolev norm. We then have
\[
\|(f_n-\E_\nu f_n)-(f-\E_\nu f)\|_{L^p(\nu)}
\le 2\|f_n-f\|_{L^p(\nu)}\longrightarrow0.
\]
Convergence in the Sobolev norm also gives $\|\nabla f_n\|_{L^p(\nu)}\to\|\nabla f\|_{L^p(\nu)}$. Taking limits in the Poincar\'e inequality for $f_n$ therefore proves the same inequality for $f$, with the same constant. See Barthe and Klartag~\citep[Proposition~27 and its proof]{bartheklartag2020spectral} for details of this type of standard density and closure argument.
\end{remark}

\begin{proof}[Proof of Proposition~\ref{ht:stability}]
Let $X\sim\mu$ and $Y\sim\widehat\mu$ be any coupling.
Recall that
\[
\tr[\Cov(\mu)]
=\E\|X-\E X\|_2^2,
\qquad
\tr[\Cov(\widehat\mu)]
=\E\|Y-\E Y\|_2^2.
\]
We can decompose the centered target vector as
\[
X-\E X
=(Y-\E Y)+\big((X-Y)-\E(X-Y)\big).
\]
The triangle inequality in $L^2$ therefore gives
\begin{align*}
\sqrt{\tr[\Cov(\mu)]}
&=\big(\E\|X-\E X\|_2^2\big)^{1/2}\\
&\le
\big(\E\|Y-\E Y\|_2^2\big)^{1/2}
+\big(\E\|(X-Y)-\E(X-Y)\|_2^2\big)^{1/2}.
\end{align*}
Moreover, we center $X-Y$ by subtracting its mean.
Expanding the squared Euclidean norm gives
\begin{align*}
\E\|(X-Y)-\E(X-Y)\|_2^2
&=\E\|X-Y\|_2^2\\
&\quad-2\,\E\langle X-Y,\E(X-Y)\rangle\\
&\quad+\|\E(X-Y)\|_2^2.
\end{align*}
Since $\E(X-Y)$ is a fixed vector, linearity of expectation yields
\begin{align*}
\E\langle X-Y,\E(X-Y)\rangle
&=\langle\E(X-Y),\E(X-Y)\rangle\\
&=\|\E(X-Y)\|_2^2.
\end{align*}
Substituting this identity into the expansion and combining terms gives
\begin{align*}
\E\|(X-Y)-\E(X-Y)\|_2^2
&=\E\|X-Y\|_2^2-\|\E(X-Y)\|_2^2\\
&\le\E\|X-Y\|_2^2.
\end{align*}
Combining these observations yields
\[
\sqrt{\tr[\Cov(\mu)]}
\le
\sqrt{\tr[\Cov(\widehat\mu)]}
+\big(\E\|X-Y\|_2^2\big)^{1/2}.
\]
By definition,
\[
W_2(\mu,\widehat\mu)
=\inf_{\substack{X\sim\mu,\;Y\sim\widehat\mu\\
\text{over all couplings}}}
\big(\E\|X-Y\|_2^2\big)^{1/2}.
\]
Since the coupling was arbitrary, taking the infimum over
couplings and using $W_2(\mu,\widehat\mu)\le\varepsilon$ gives
\[
\sqrt{\tr[\Cov(\mu)]}
\le
\sqrt{\tr[\Cov(\widehat\mu)]}
+W_2(\mu,\widehat\mu)
\le
\sqrt{\tr[\Cov(\widehat\mu)]}+\varepsilon.
\]
Rearranging and using the nonnegativity of
$\sqrt{\tr[\Cov(\widehat\mu)]}$, we obtain
\[
\sqrt{\tr[\Cov(\widehat\mu)]}
\ge
\big(\sqrt{\tr[\Cov(\mu)]}-\varepsilon\big)_+.
\]
Finally, since $\widehat\mu=T_\#\gamma$, the trace bound in
Theorem~\ref{lem:vector_poincare} implies
\[
\big(\E_\gamma\|J_T\|_F^2\big)^{1/2}
\ge
\sqrt{\tr[\Cov(\widehat\mu)]}
\ge
\big(\sqrt{\tr[\Cov(\mu)]}-\varepsilon\big)_+,
\]
which proves the claim.
\end{proof}

\section{Proofs for Section~\ref{sec:multimodal}}
\label{mm:appendix}

\noindent For a nonempty closed set $A\subseteq\mathbb R^m$, let
\[
\operatorname{dist}(y,A)=\inf_{a\in A}\|y-a\|_2
\]
denote the Euclidean distance from $y$ to $A$.

We use the Gaussian isoperimetric inequality \citep{borell1975} and
the coarea formula \citep[Section~3.4]{evansgariepy2015}. Together, these imply that every bounded locally Lipschitz function $g:\R^d\to\R$ satisfies
\begin{equation}
\label{mm:gaussian-functional}
\int_{\R^d}\|\nabla g\|_2\,d\gamma
\geq\int_{\R} I\big(\gamma(\{g>t\})\big)\,dt,
\end{equation}
where the integral on the left may be infinite. An explicit formulation appears in Caselles, Miranda Jr., and Novaga~\citep[equation~(11)]{caselles2010totalvariation}.

\begin{proof}[Proof of Theorem~\ref{mm:theorem}]
Define
\[
f(z)=\min\{\operatorname{dist}(T(z),A),\Delta\}.
\]
The truncated distance is $1$-Lipschitz, so $f$ is locally Lipschitz and $\|\nabla f\|_2\leq\|J_T\|_{\mathrm{op}}$ almost everywhere. Since $0\leq f\leq\Delta$, its gradient vanishes at differentiability points where $f=0$ or $f=\Delta$. Moreover, $\{0<f<\Delta\}=T^{-1}(U)$. Consequently,
\begin{equation}
\label{mm:distance-gradient}
\|\nabla f\|_2
\leq\mathbf 1_{T^{-1}(U)}\|J_T\|_{\mathrm{op}}
\quad\text{almost everywhere}.
\end{equation}

For $0<t<\Delta$, we have
\[
\{f>t\}=T^{-1}((A^t)^c),\qquad
\gamma(\{f>t\})=1-P(A^t).
\]
Applying \eqref{mm:gaussian-functional} and using the symmetry of $I$, we obtain
\begin{align*}
\int_{T^{-1}(U)}\|J_T\|_{\mathrm{op}}\,d\gamma
&\geq\int_{\R^d}\|\nabla f\|_2\,d\gamma\\
&\geq\int_0^\Delta I\big(\gamma(\{f>t\})\big)\,dt\\
&=\int_0^\Delta I\big(P(A^t)\big)\,dt.
\end{align*}
This proves the first inequality.

The sets $A$ and $(A^\Delta)^c$ are disjoint, so $b\leq1/2$. The inclusions $A\subseteq A^t\subseteq A^\Delta$ give
\[
b\leq P(A^t)\leq1-b.
\]
By symmetry and monotonicity of $I$ on $[0,1/2]$, we therefore have $I(P(A^t))\geq I(b)$. Integrating over $0<t<\Delta$ proves the second inequality.
\end{proof}

\begin{proof}[Proof of Corollary~\ref{mm:corollary}]
The definition of total variation and the assumed probability bounds for $\mu$ give
\[
\begin{aligned}
P(A)&\geq\mu(A)-\varepsilon\geq w_1-\varepsilon,\\
P((A^\Delta)^c)&\geq\mu((A^\Delta)^c)-\varepsilon\geq w_2-\varepsilon,\\
P(U)&\leq\mu(U)+\varepsilon\leq\delta+\varepsilon.
\end{aligned}
\]
Thus the quantity $b$ in Theorem~\ref{mm:theorem} satisfies $b\geq w-\varepsilon>0$. Applying Theorem~\ref{mm:theorem} and using monotonicity of $I$ on $[0,1/2]$, we obtain
\[
\int_{T^{-1}(U)}\|J_T\|_{\mathrm{op}}\,d\gamma
\geq\Delta I(w-\varepsilon)>0.
\]
Together with $P(U)\leq\delta+\varepsilon$, this proves
bound~\eqref{mm:localized-mass}.

The integral is strictly positive, so $T^{-1}(U)$ has positive Gaussian measure. Therefore
\[
P(U)=\gamma(T^{-1}(U))>0.
\]

For $\delta+\varepsilon>0$ and $q>1$, H\"older's inequality gives
\[
\begin{aligned}
\left(\int_{T^{-1}(U)}\|J_T\|_{\mathrm{op}}^q\,d\gamma\right)^{1/q}
&\geq\frac{\displaystyle\int_{T^{-1}(U)}\|J_T\|_{\mathrm{op}}\,d\gamma}{P(U)^{1-1/q}}\\
&\geq\frac{\Delta I(w-\varepsilon)}{(\delta+\varepsilon)^{1-1/q}}.
\end{aligned}
\]
The global moment dominates the localized moment by nonnegativity. These inequalities also hold when the moment is infinite.
\end{proof}

{\hypersetup{linkcolor=blue,filecolor=blue}
\begin{proof}[Proof of Proposition~\ref{mm:logconcave}]
Define
\[
f(z)=\min\{\operatorname{dist}(T(z),A),\Delta\}.
\]
Since $\nu$ has a density, the gradient estimate \eqref{mm:distance-gradient} established in the proof of Theorem~\ref{mm:theorem} holds $\nu$-almost everywhere. Consequently,
\[
\int_{T^{-1}(U)}\|J_T\|_{\mathrm{op}}\,d\nu
\geq\E_\nu\|\nabla f\|_2.
\]

The function $f$ equals $0$ on $T^{-1}(A)$ and $\Delta$ on $T^{-1}((A^\Delta)^c)$. Each of these sets has probability at least $b$. Therefore, for every $c\in\R$,
\[
\E_\nu|f-c|\geq b|c|+b|\Delta-c|\geq b\Delta.
\]
Applying the Cheeger inequality in \eqref{mm:cheeger-functional} to the bounded locally Lipschitz function $f$ gives
\[
\E_\nu\|\nabla f\|_2
\geq h_\nu\inf_{c\in\R}\E_\nu|f-c|
\geq h_\nu\Delta b.
\]
Combining the gradient estimate \eqref{mm:distance-gradient} with the
Cheeger inequality \eqref{mm:cheeger-functional} and
\[
\inf_{c\in\R}\E_\nu|f-c|\geq b\Delta
\]
proves the localized
bound~\eqref{mm:logconcave-localized}.

To prove the claims concerning approximation of the target distribution $\mu$, assume
\[
\operatorname{TV}(P,\mu)\leq\varepsilon<w=\min\{w_1,w_2\}.
\]
Applying the definition of total variation to the three sets $A$, $(A^\Delta)^c$, and $U$, and using the assumed probability bounds for $\mu$, yields
\[
\begin{aligned}
P(A)&\geq\mu(A)-\varepsilon\geq w_1-\varepsilon,\\
P((A^\Delta)^c)&\geq\mu((A^\Delta)^c)-\varepsilon\geq w_2-\varepsilon,\\
P(U)&\leq\mu(U)+\varepsilon\leq\delta+\varepsilon.
\end{aligned}
\]
Hence $b\geq w-\varepsilon>0$, and the localized bound
\eqref{mm:logconcave-localized} gives
\[
\int_{T^{-1}(U)}\|J_T\|_{\mathrm{op}}\,d\nu
\geq h_\nu\Delta(w-\varepsilon)>0.
\]
The integral is strictly positive, so $T^{-1}(U)$ has positive $\nu$-measure. Therefore
\[
P(U)=\nu(T^{-1}(U))>0.
\]

For $\delta+\varepsilon>0$ and $q>1$, H\"older's inequality gives
\[
\begin{aligned}
\left(\int_{T^{-1}(U)}\|J_T\|_{\mathrm{op}}^q\,d\nu\right)^{1/q}
&\geq\frac{\displaystyle\int_{T^{-1}(U)}\|J_T\|_{\mathrm{op}}\,d\nu}{P(U)^{1-1/q}}\\
&\geq\frac{h_\nu\Delta(w-\varepsilon)}{(\delta+\varepsilon)^{1-1/q}}.
\end{aligned}
\]
Nonnegativity gives the global moment bound, completing the proof.
\end{proof}
}

\subsection{Calculations for Example~\ref{mm:example}}
\label{mm:example-proof}

We verify the claims in Example~\ref{mm:example}. Throughout, $a>0$
is fixed and $0<\sigma\le1$. We first show that the target moments
remain bounded as $\sigma\downarrow0$. We then show that the probability
between the modes tends to zero and apply Corollary~\ref{mm:corollary}.
Finally, we verify the derivative formula for the monotone transport.

To bound the target moments, let $B$ take the values $-1$ and $1$ with equal probabilities,
independently of a standard Gaussian random variable $G$. Then
\[
X_\sigma=aB+\sigma G
\]
has distribution $\mu_\sigma$. Since $|B|=1$ and $\sigma\le1$,
\[
|X_\sigma|\le a+|G|.
\]
Therefore, for every fixed $r>0$,
\[
\sup_{0<\sigma\le1}\E|X_\sigma|^r
\le \E(a+|G|)^r<\infty.
\]
The last expectation is finite because Gaussian random variables have
finite moments of every order. Thus every fixed absolute moment of
the target remains bounded as $\sigma\downarrow0$.

To compute the probability assigned to the separating region
and apply Corollary~\ref{mm:corollary}, take
\[
A=(-\infty,-a/2],\qquad \Delta=a,\qquad U=(-a/2,a/2).
\]
For the Gaussian component with mean $-a$,
\[
\mathbb P(-a+\sigma G\in U)
=\mathbb P\!\left(\frac{a}{2\sigma}<G<\frac{3a}{2\sigma}\right).
\]
The component with mean $a$ gives the same probability by symmetry.
Since the mixture weights sum to one, the probability
$\delta_\sigma=\mu_\sigma(U)$ is given by
\begin{align*}
\delta_\sigma
&=\Phi\!\left(\frac{3a}{2\sigma}\right)
  -\Phi\!\left(\frac{a}{2\sigma}\right)\\
&\le 1-\Phi\!\left(\frac{a}{2\sigma}\right)
\longrightarrow0.
\end{align*}
The limit follows because $a/(2\sigma)\to\infty$.

By symmetry and continuity of the target distribution, each side of
the valley has probability
\[
w_\sigma=\frac{1-\delta_\sigma}{2}\longrightarrow\frac12.
\]
For any locally Lipschitz exact transport $T$ from a standard Gaussian
source to $\mu_\sigma$, Corollary~\ref{mm:corollary} applies with
$\varepsilon=0$ and gives
\[
\big(\E_\gamma\|J_T\|_{\mathrm{op}}^q\big)^{1/q}
\ge
\frac{aI(w_\sigma)}{\delta_\sigma^{\,1-1/q}},
\qquad q>1.
\]
For every fixed $q>1$, the numerator tends to $aI(1/2)>0$, while
the denominator tends to zero. Hence the lower bound diverges, even
though every fixed target moment remains bounded.

Finally, consider the univariate monotone transport
\[
T_\sigma=F_\sigma^{-1}\circ\Phi,
\]
where $F_\sigma$ and $f_\sigma$ are the cumulative distribution function
and density of $\mu_\sigma$. The density is smooth and strictly
positive, so $F_\sigma$ has a smooth inverse on $(0,1)$.
Consequently, $T_\sigma$ is a locally Lipschitz exact transport.

Differentiating $F_\sigma(T_\sigma(z))=\Phi(z)$ gives
\[
f_\sigma(T_\sigma(z))T_\sigma'(z)=\varphi(z),
\]
and therefore
\[
T_\sigma'(z)=\frac{\varphi(z)}{f_\sigma(T_\sigma(z))}.
\]
Symmetry gives $F_\sigma(0)=1/2$, so $T_\sigma(0)=0$.
The target density at zero is
\begin{align*}
f_\sigma(0)
&=\frac{1}{2\sigma}\varphi(a/\sigma)
  +\frac{1}{2\sigma}\varphi(-a/\sigma)\\
&=\frac{1}{\sigma}\varphi(a/\sigma).
\end{align*}
Substituting into the derivative formula yields
\[
T_\sigma'(0)
=\frac{\sigma\varphi(0)}{\varphi(a/\sigma)}
=\sigma\exp\!\left(\frac{a^2}{2\sigma^2}\right)
\longrightarrow\infty.
\]
This verifies the derivative formula in Example~\ref{mm:example}.

\section{Proofs for Section~\ref{sec:transformations}}
\label{tf:app:proofs}

\subsection{Proofs for Bijective Reparametrizations}

\begin{proof}[Proof of Proposition~\ref{tf:prop:tail-transfer}]
We first establish the exponential tail rate and then use it to
prove the moment claims.

Since $\psi_s$ is odd and strictly increasing, its inverse gives
\[
\mathbb P(|Y|>y)
=
\mathbb P(|X|>s\sinh y)
=
(s\sinh y)^{-\alpha}L(s\sinh y)
\]
for all sufficiently large $y$. Taking logarithms and dividing by
$y$, we obtain
\[
\frac{-\log\mathbb P(|Y|>y)}{y}
=
\alpha\frac{\log(s\sinh y)}{y}
-
\frac{\log L(s\sinh y)}{y}.
\]

To evaluate the first term, use
\[
s\sinh y
=
\frac{s}{2}(e^y-e^{-y})
=
\frac{s}{2}e^y(1-e^{-2y}),
\qquad y>0.
\]
It follows that
\[
\log(s\sinh y)
=
y+\log(s/2)+\log(1-e^{-2y}).
\]
We therefore obtain
\[
\frac{\log(s\sinh y)}{y}\longrightarrow 1
\qquad\text{as }y\to\infty.
\]

A positive slowly varying function satisfies
\[
\lim_{x\to\infty}\frac{\log L(x)}{\log x}=0
\]
\citep{bingham1987regular}. Since $s\sinh y\to\infty$ as
$y\to\infty$, we consequently have
\[
\frac{\log L(s\sinh y)}{y}
=
\frac{\log L(s\sinh y)}{\log(s\sinh y)}
\frac{\log(s\sinh y)}{y}
\longrightarrow 0
\qquad\text{as }y\to\infty.
\]
Substituting these two limits proves the tail-rate formula
\eqref{tf:eq:tail-rate}.

We next show that $\E e^{t|Y|}<\infty$ for $0<t<\alpha$.
Fix such a $t$ and choose $b$ with $t<b<\alpha$.
The tail-rate limit \eqref{tf:eq:tail-rate} implies that there is
a finite $y_0\ge0$ such that
\[
\mathbb P(|Y|>y)\le e^{-by},
\qquad y\ge y_0.
\]
Using the equality
\[
e^{t|Y|}
=
1+t\int_0^{|Y|}e^{ty}\,dy
\]
and applying Tonelli's theorem, we obtain
\[
\begin{aligned}
\E e^{t|Y|}
&=
1+t\int_0^\infty e^{ty}\mathbb P(|Y|>y)\,dy\\
&\le
1+t\int_0^{y_0}e^{ty}\,dy
+t\int_{y_0}^\infty e^{-(b-t)y}\,dy.
\end{aligned}
\]
In the last expression, the first integral is finite because its
integrand is continuous on the bounded interval $[0,y_0]$.
The second integral is also finite, since
\[
\int_{y_0}^\infty e^{-(b-t)y}\,dy
=
\frac{e^{-(b-t)y_0}}{b-t}
<\infty.
\]
Thus $\E e^{t|Y|}<\infty$.

Finally, to establish finiteness of every polynomial moment,
fix $p>0$ and choose any $t\in(0,\alpha)$.
The function $u^p e^{-tu}$ is bounded on $[0,\infty)$, so there
is a finite constant $M_{p,t}$ such that
\[
u^p\le M_{p,t}e^{tu},
\qquad u\ge0.
\]
It follows that
\[
\E|Y|^p
\le M_{p,t}\E e^{t|Y|}
<\infty.
\]
This proves the polynomial-moment claim.
\end{proof}

\subsection{Calculations for Example~\ref{tf:ex:student}}
\label{tf:supp:student}

We verify the transformed density in
Example~\ref{tf:ex:student} and prove the finite expected squared
derivative statement~\eqref{tf:eq:student-energy}.
The calculation concerns a fixed $0<\kappa\le2$.

Write $p_\kappa$ for the Student-$t$ density.
The inverse transformation and its derivative are
\[
x=\sqrt{\kappa}\sinh y,
\qquad
\frac{dx}{dy}=\sqrt{\kappa}\cosh y>0.
\]
Since this transformation is strictly increasing,
\[
F_{\tau_\kappa}(y)
=
\mathbb P(Y\le y)
=
\mathbb P(X\le\sqrt{\kappa}\sinh y).
\]
Differentiating with respect to $y$ and substituting the
Student-$t$ density gives
\[
\begin{aligned}
f_\kappa(y)
&=
p_\kappa(\sqrt{\kappa}\sinh y)\sqrt{\kappa}\cosh y\\
&=
\frac{\Gamma((\kappa+1)/2)}
{\sqrt{\kappa\pi}\,\Gamma(\kappa/2)}
\left(
1+\frac{(\sqrt{\kappa}\sinh y)^2}{\kappa}
\right)^{-(\kappa+1)/2}
\sqrt{\kappa}\cosh y\\
&=
\frac{\Gamma((\kappa+1)/2)}
{\sqrt{\pi}\,\Gamma(\kappa/2)}
(1+\sinh^2 y)^{-(\kappa+1)/2}\cosh y\\
&=
\frac{\Gamma((\kappa+1)/2)}
{\sqrt{\pi}\,\Gamma(\kappa/2)}
(\cosh^2 y)^{-(\kappa+1)/2}\cosh y\\
&=
\frac{\Gamma((\kappa+1)/2)}
{\sqrt{\pi}\,\Gamma(\kappa/2)}
(\cosh y)^{-(\kappa+1)}\cosh y\\
&=
\frac{\Gamma((\kappa+1)/2)}
{\sqrt{\pi}\,\Gamma(\kappa/2)}
(\cosh y)^{-\kappa}.
\end{aligned}
\]
Here the factors $\sqrt{\kappa}$ cancel, and we use
$1+\sinh^2 y=\cosh^2 y$.
Since $\cosh y>0$, the remaining powers combine as
\[
(\cosh^2 y)^{-(\kappa+1)/2}\cosh y
=
(\cosh y)^{-(\kappa+1)}\cosh y
=
(\cosh y)^{-\kappa}.
\]
This proves~\eqref{tf:eq:student-density}.

To bound the derivative of the monotone transport map, we first
compare the upper-tail probability of $\tau_\kappa$ with its
density.
Write
\[
f_\kappa(y)=c_\kappa(\cosh y)^{-\kappa},
\qquad
Q_\kappa(y)
=
1-F_{\tau_\kappa}(y)
=
\int_y^\infty f_\kappa(v)\,dv,
\]
where $c_\kappa>0$ is the normalizing constant in
\eqref{tf:eq:student-density}.
For $y\ge0$, the inequalities
\[
\frac{e^y}{2}\le\cosh y\le e^y
\]
give, upon taking reciprocals,
\[
e^{-y}\le\frac{1}{\cosh y}\le2e^{-y}.
\]
Raising each term to the positive power $\kappa$ and multiplying
by $c_\kappa$ therefore yields
\[
c_\kappa e^{-\kappa y}
\le
f_\kappa(y)
\le
2^\kappa c_\kappa e^{-\kappa y},
\qquad y\ge0.
\]
Integrating the upper bound over $[y,\infty)$ gives
\[
Q_\kappa(y)
\le
\frac{2^\kappa c_\kappa}{\kappa}e^{-\kappa y}.
\]
Dividing by the lower bound on $f_\kappa(y)$, we obtain
\[
\frac{Q_\kappa(y)}{f_\kappa(y)}
\le
\frac{2^\kappa}{\kappa},
\qquad y\ge0.
\]

We now apply this bound to
$S_\kappa=F_{\tau_\kappa}^{-1}\circ\Phi$.
Let $\varphi$ denote the standard Gaussian density.
Since $f_\kappa$ is smooth and strictly positive,
$S_\kappa$ is smooth, and differentiating
\[
F_{\tau_\kappa}(S_\kappa(z))=\Phi(z)
\]
gives
\[
S_\kappa'(z)
=
\frac{\varphi(z)}{f_\kappa(S_\kappa(z))}.
\]
Both $\varphi$ and $f_\kappa$ are even.
We first express this property in terms of their CDFs.
For the transformed target, substituting $u=-v$ gives
\[
\begin{aligned}
F_{\tau_\kappa}(-y)
&=
\int_{-\infty}^{-y}f_\kappa(v)\,dv\\
&=
\int_y^\infty f_\kappa(-u)\,du\\
&=
\int_y^\infty f_\kappa(u)\,du\\
&=
1-F_{\tau_\kappa}(y).
\end{aligned}
\]
The same argument for the even Gaussian density gives
\[
\Phi(-z)=1-\Phi(z).
\]

Using $F_{\tau_\kappa}(S_\kappa(z))=\Phi(z)$, we therefore obtain
\[
\begin{aligned}
F_{\tau_\kappa}(S_\kappa(-z))
&=
\Phi(-z)\\
&=
1-\Phi(z)\\
&=
1-F_{\tau_\kappa}(S_\kappa(z))\\
&=
F_{\tau_\kappa}(-S_\kappa(z)).
\end{aligned}
\]
Since $f_\kappa>0$, the CDF $F_{\tau_\kappa}$ is strictly
increasing.
Equality of these CDF values therefore implies
\[
S_\kappa(-z)=-S_\kappa(z),
\]
so the monotone transport map $S_\kappa$ is odd and its
derivative $S_\kappa'$ is even.

Oddness also gives $S_\kappa(0)=0$.
Moreover, $S_\kappa=F_{\tau_\kappa}^{-1}\circ\Phi$ is strictly
increasing because both $F_{\tau_\kappa}^{-1}$ and $\Phi$
are strictly increasing.
Consequently,
\[
S_\kappa(z)>S_\kappa(0)=0,
\qquad z>0.
\]
In particular, $S_\kappa(z)\ge0$ for $z\ge0$, so we can apply
the bound on $Q_\kappa(y)/f_\kappa(y)$ at $y=S_\kappa(z)$.
To do so, we also use
\[
Q_\kappa(S_\kappa(z))
=
1-F_{\tau_\kappa}(S_\kappa(z))
=
1-\Phi(z).
\]
Consequently, for $z\ge0$,
\[
\begin{aligned}
S_\kappa'(z)
&=
\frac{\varphi(z)}{1-\Phi(z)}
\frac{Q_\kappa(S_\kappa(z))}
{f_\kappa(S_\kappa(z))}\\
&\le
\frac{2^\kappa}{\kappa}
\frac{\varphi(z)}{1-\Phi(z)}.
\end{aligned}
\]

A standard lower bound for the Gaussian Mills ratio is
\[
\frac{1-\Phi(z)}{\varphi(z)}
\ge
\frac{z}{1+z^2},
\qquad z>0.
\]
Taking reciprocals, for $z\ge1$ we obtain
\[
\frac{\varphi(z)}{1-\Phi(z)}
\le z+\frac1z
\le2z.
\]
For $0\le z\le1$, we instead use
$\varphi(z)\le\varphi(0)$ and
$1-\Phi(z)\ge1-\Phi(1)>0$, which give
\[
\frac{\varphi(z)}{1-\Phi(z)}
\le
\frac{\varphi(0)}{1-\Phi(1)}
<\infty.
\]
Combining the bounds for these two ranges, and recalling that
$\kappa$ is fixed, we obtain a finite constant $C_\kappa$ such that
\[
S_\kappa'(z)\le C_\kappa(1+z),
\qquad z\ge0.
\]
Since $S_\kappa'$ is even, this bound extends to
\[
S_\kappa'(z)\le C_\kappa(1+|z|),
\qquad z\in\R.
\]
For $Z\sim N(0,1)$, squaring and taking expectation gives
\[
\E[(S_\kappa'(Z))^2]
\le
C_\kappa^2\E[(1+|Z|)^2]
<\infty.
\]
This proves~\eqref{tf:eq:student-energy}.

{\hypersetup{linkcolor=blue,filecolor=blue}
\subsection{Reparametrizations for Multimodal Target Densities}
\label{supp:component-means}

We provide the details for the reparametrization in
Section~\ref{tf:subsec:component-means}.
We first verify that $\Psi$ is an invertible change of coordinates.
We then prove the derivative comparisons in
Propositions~\ref{tf:prop:component-locations}
and~\ref{cm:prop:original-growth}.

\begin{proposition}[Invertibility]
\label{cm:prop:geometry}
The map $\Psi$ in \eqref{tf:eq:component-means-map} is a bijection
from $\R$ onto $\R$ such that both $\Psi$ and $\Psi^{-1}$ are
$C^1$, and satisfies $0<\Psi'\le1$.
\end{proposition}

\begin{proof}
The distance function
\[
d(x)=\min_i|x-c_i|
\]
is continuous.
The formula for $a$ on $[L_\Psi,U_\Psi]$ in
\eqref{tf:eq:component-means-derivative} is therefore continuous.
At the endpoints $L_\Psi=c_1-r$ and $U_\Psi=c_K+r$,
the value of $d$ is $r$, so this formula gives $a=1$.
Since $a(x)=1$ for $x<L_\Psi$ and $x>U_\Psi$,
the function $a$ is continuous at both endpoints and hence on $\R$.
Its definition also gives
\[
0<a(x)\le1,
\qquad x\in\R.
\]

By the fundamental theorem of calculus,
\[
\Psi(x)=\int_0^x a(t)\,dt
\]
is $C^1$ with $\Psi'=a>0$.
Thus $\Psi$ is strictly increasing.
Outside $[L_\Psi,U_\Psi]$, its derivative equals one,
so $\Psi(x)$ equals $x$ plus a constant on each tail.
Consequently,
\[
\lim_{x\to-\infty}\Psi(x)=-\infty,
\qquad
\lim_{x\to\infty}\Psi(x)=\infty.
\]
Together with continuity and strict monotonicity, these limits
show that $\Psi$ is a bijection onto $\R$.

The inverse function theorem gives
\[
(\Psi^{-1})'(y)
=
\frac{1}{a(\Psi^{-1}(y))}.
\]
The right-hand side is continuous because $a$ is continuous
and strictly positive.
Hence $\Psi^{-1}$ is also $C^1$.
\end{proof}

Having established invertibility, we now compare the monotone
transport maps from the standard Gaussian to the original and
transformed targets.
Proposition~\ref{tf:prop:component-locations} gives a derivative
bound for the map to the transformed target that is independent
of the component separation $D$.
Proposition~\ref{cm:prop:original-growth} shows that the expected
squared derivative of the map to the original target grows
rapidly with $D$.

Let $\mu_D$ have density
\begin{equation}
\label{cm:eq:mixture}
p_D(x)
=
\tfrac12\varphi(x+D/2)+\tfrac12\varphi(x-D/2),
\qquad D>0,
\end{equation}
where $\varphi$ and $\Phi$ are the standard Gaussian density
and CDF.
Write
\[
m_1=-D/2,\qquad m_2=D/2
\]
for the component means, and let $F_D$ be the CDF of $\mu_D$.

For the bound in
Proposition~\ref{tf:prop:component-locations},
fix $r>0$, $0<h<1$, and $0\le\eta\le r$.
Suppose the reference component locations satisfy
\[
c_1<c_2,\qquad |c_i-m_i|\le\eta,\quad i=1,2.
\]
Construct $\Psi$ using
\eqref{tf:eq:component-means-map} and
\eqref{tf:eq:component-means-derivative}, and define
\begin{equation}
\label{cm:eq:transports}
T_D(z)=F_D^{-1}(\Phi(z)),
\qquad
S_D(z)=\Psi(T_D(z)).
\end{equation}
For $Z\sim N(0,1)$, the variable $T_D(Z)$ has distribution
$\mu_D$, while $S_D(Z)$ has distribution $\Psi_\#\mu_D$.
The CDF $F_D$ specifies the exact transports used in this
comparison.
Constructing $\Psi$ requires only the reference locations
and the parameters $r,h$.

To obtain formulas for the transport derivatives, differentiate
\[
F_D(T_D(z))=\Phi(z).
\]
Since $p_D$ is smooth and strictly positive, this gives
\[
p_D(T_D(z))T_D'(z)=\varphi(z).
\]
Let
\[
I(u)=\varphi(\Phi^{-1}(u)),
\qquad 0<u<1.
\]
For $x=T_D(z)$, we have $z=\Phi^{-1}(F_D(x))$.
Applying the chain rule to $S_D=\Psi\circ T_D$ therefore yields
\begin{equation}
\label{cm:eq:gradient}
T_D'(z)
=
\frac{I(F_D(x))}{p_D(x)},
\qquad
S_D'(z)
=
a(x)\frac{I(F_D(x))}{p_D(x)}.
\end{equation}
These formulas identify the factor $a(x)$ through which the
change of coordinates modifies the transport derivative.

We will use the elementary properties
\[
I(u)\le\frac{1}{\sqrt{2\pi}},
\qquad
I(1-u)=I(u),
\qquad
I'(u)=-\Phi^{-1}(u).
\]
In particular, $I$ is increasing on $(0,1/2]$ and decreasing
on $[1/2,1)$.

\begin{proof}[Proof of Proposition~\ref{tf:prop:component-locations}]
Fix $z\in\R$ and write $x=T_D(z)$.
We bound $S_D'(z)$ separately inside and outside
$[L_\Psi,U_\Psi]$.
Inside this interval, the factor $a(x)$ offsets the small mixture
density in the denominator of \eqref{cm:eq:gradient}.
Outside the interval, $a(x)=1$, and we bound the original
transport derivative directly.

First suppose $x\in[L_\Psi,U_\Psi]$.
Choose a nearest reference location $c_j$, so that
$|x-c_j|=d(x)$.
The reference-location error bound gives
\[
|x-m_j|
\le
|x-c_j|+|c_j-m_j|
\le
d(x)+\eta.
\]
The corresponding Gaussian component therefore gives
\[
p_D(x)
\ge
\tfrac12\varphi(x-m_j)
\ge
\frac{1}{2\sqrt{2\pi}}
\exp\left(-\frac{(d(x)+\eta)^2}{2}\right).
\]
Combining this bound with $I\le1/\sqrt{2\pi}$,
\eqref{cm:eq:gradient}, and the formula for $a$ in
\eqref{tf:eq:component-means-derivative}, we obtain
\[
S_D'(z)
\le
2\exp\left(
\frac{(d(x)+\eta)^2}{2}
-
\frac{(d(x)-r)_+^2}{2h^2}
\right).
\]

We next bound the exponent independently of $d(x)$.
If $d(x)\le r$, then
\[
\frac{(d(x)+\eta)^2}{2}
-
\frac{(d(x)-r)_+^2}{2h^2}
\le
\frac{(r+\eta)^2}{2}
\le
\frac{(r+\eta)^2}{2(1-h^2)}.
\]
If $d(x)>r$, write $d(x)=r+u$ with $u>0$.
Completing the square gives
\[
\begin{aligned}
\frac{(u+r+\eta)^2}{2}-\frac{u^2}{2h^2}
&=
\frac{(r+\eta)^2}{2(1-h^2)}\\
&\quad-
\frac{1-h^2}{2h^2}
\left(
u-\frac{h^2(r+\eta)}{1-h^2}
\right)^2\\
&\le
\frac{(r+\eta)^2}{2(1-h^2)},
\end{aligned}
\]
where the inequality uses $0<h<1$.
Both cases therefore give
\[
S_D'(z)
\le
2\exp\left(
\frac{(r+\eta)^2}{2(1-h^2)}
\right)
=:C
\]
for $x\in[L_\Psi,U_\Psi]$.
This constant depends only on $r,h,\eta$ and is independent of $D$.

Now suppose $x<L_\Psi=c_1-r$.
Since $c_1-m_1\le\eta\le r$, we have
$L_\Psi\le m_1<m_2$.
Thus $x$ lies to the left of both component means, and
\[
F_D(x)
=
\tfrac12\Phi(x-m_1)+\tfrac12\Phi(x-m_2)
\le
\Phi(x-m_1)
\le\tfrac12.
\]
Because $I$ is increasing on $(0,1/2]$,
\[
I(F_D(x))
\le
I(\Phi(x-m_1))
=
\varphi(x-m_1).
\]
Together with $p_D(x)\ge\varphi(x-m_1)/2$ and $a(x)=1$,
the formula for $S_D'$ in \eqref{cm:eq:gradient} gives
\[
S_D'(z)\le2\le C.
\]

Similarly, suppose $x>U_\Psi=c_2+r$.
Since $m_2-c_2\le\eta\le r$, we have
$m_1<m_2\le U_\Psi<x$.
Consequently,
\[
F_D(x)\ge\Phi(x-m_2)\ge\tfrac12.
\]
Because $I$ is decreasing on $[1/2,1)$,
\[
I(F_D(x))
\le
I(\Phi(x-m_2))
=
\varphi(x-m_2).
\]
Using $p_D(x)\ge\varphi(x-m_2)/2$ and $a(x)=1$ in
\eqref{cm:eq:gradient} again yields
\[
S_D'(z)\le2\le C.
\]

All factors in the expression for $S_D'$ in
\eqref{cm:eq:gradient} are strictly positive.
Combining the three regions proves
\[
0<S_D'(z)\le C,
\qquad z\in\R.
\]
Squaring and taking expectation gives
\[
\E[S_D'(Z)^2]\le C^2,
\qquad Z\sim N(0,1),
\]
as required.
\end{proof}

We next prove the growth bound for the original monotone
transport map $T_D$.
Near the midpoint between the component means, the density
$p_D(x)$ becomes exponentially small as $D$ increases,
while $I(F_D(x))$ remains bounded below.
We use an interval around this midpoint to obtain a lower bound
on the expected squared derivative.

\begin{proof}[Proof of Proposition~\ref{cm:prop:original-growth}]
Since $T_D(Z)$ has density $p_D$, the formula for $T_D'$ in
\eqref{cm:eq:gradient} gives
\[
\begin{aligned}
\E[T_D'(Z)^2]
&=
\int_\R
\left(\frac{I(F_D(x))}{p_D(x)}\right)^2p_D(x)\,dx\\
&=
\int_\R
\frac{I(F_D(x))^2}{p_D(x)}\,dx,
\end{aligned}
\]
allowing either side to be infinite.
We obtain a lower bound by restricting the last integral to
$[-1/D,1/D]$.

For $D\ge2$, this interval lies between the component means.
Thus, for every $x$ in the interval,
\[
\frac12\le\Phi(x-m_1)\le1,
\qquad
0\le\Phi(x-m_2)\le\frac12.
\]
Averaging the two component CDFs gives
\[
\frac14\le F_D(x)\le\frac34.
\]
The symmetry and monotonicity of $I$ therefore imply
\[
I(F_D(x))\ge I(1/4)>0.
\]

On the same interval, the distance from $x$ to either component
mean is at least $D/2-1/D$.
Since the standard Gaussian density decreases with distance
from zero,
\[
\begin{aligned}
p_D(x)
&\le
\frac{1}{\sqrt{2\pi}}
\exp\left[-\frac12\left(\frac D2-\frac1D\right)^2\right]\\
&=
\frac{1}{\sqrt{2\pi}}
\exp\left(-\frac{D^2}{8}+\frac12-\frac{1}{2D^2}\right)\\
&\le
\frac{e^{1/2}}{\sqrt{2\pi}}e^{-D^2/8}.
\end{aligned}
\]

Combining the lower bound on $I(F_D(x))$ with the upper bound
on $p_D(x)$, and using the interval length $2/D$, we obtain
\[
\begin{aligned}
\E[T_D'(Z)^2]
&\ge
\int_{-1/D}^{1/D}
\frac{I(F_D(x))^2}{p_D(x)}\,dx\\
&\ge
\frac{2}{D}\,I(1/4)^2\sqrt{2\pi}\,e^{-1/2}e^{D^2/8}\\
&=
\frac{c}{D}e^{D^2/8},
\end{aligned}
\]
where
$c=2I(1/4)^2\sqrt{2\pi}\,e^{-1/2}>0$
is independent of $D$.
This proves the lower bound \eqref{cm:eq:original-growth}.
\end{proof}

}

\section{Experimental Details and Additional Experiments}
\label{app:additional-sims}

\subsection{Details of the Heavy-Tail Illustrations}
\label{supp:heavy-tail-details}

These distribution families are widely used in many fields such as climatology, hydrology, and finance \citep{coles2001introduction, katz2002statistics, bradley2003financial}, making them natural and representative test cases. They are also particularly well-suited for illustrating our theory for several reasons. First, a transport map from a Gaussian source can be constructed explicitly for these targets, enabling direct
inspection of their gradients and Jacobians. This allows us to numerically assess the tightness and practical relevance of the theoretical lower bounds derived in Section \ref{sec:gradient-bounds}.
Second, the tail behavior of these families
is governed by interpretable parameters (i.e., the degrees of freedom for the Student-$t$ family and the shape parameters for the Pareto and Fr\'echet families), providing explicit control over moment
growth and allowing us to systematically investigate varying degrees of tail heaviness.

In the multivariate case, when pushing forward a standard Gaussian in $\R^d$ to a product distribution also in $\R^d$, the pushforward map can be applied coordinate-wise. If $d=2$, this is given by $$
T(z_1,z_2)
=
\big(F_{\mu}^{-1}(F_{\nu}(z_1)), F_{\mu}^{-1}(F_{\nu}(z_2))\big).
$$
The Jacobian of $T$ is thus also diagonal. We adopt this construction for the explicit heavy-tail simulations below.

As target distributions, we consider
three canonical families of heavy-tailed laws: (a) Student-$t$
distributions, (b) Pareto distributions, and (c) Fr\'echet
distributions. Here, all target measures are product measures with independent
coordinates. For ease of visualization, we present results in the two-dimensional setting.
Since all target distributions are product measures, this entails no loss of generality,
and the behavior observed in two dimensions is representative of higher-dimensional
cases.

For each parameter setting, we plot both sides of the trace bound Poincar\'e inequality given in part (iii) of Theorem \ref{lem:vector_poincare}. Specifically, the lower bound (left-hand side of the inequality) corresponds to the trace of the covariance matrix of the target distribution, which admits a closed-form expression for the distributions
considered. The right-hand side of the inequality is given by the sum of the expected squared norm of the Jacobian diagonals of the transport map under the Gaussian source measure. We compute this
expectation numerically using Gauss--Hermite quadrature. Plotting
these quantities together allows us to directly assess the tightness and practical
relevance of our theoretical bounds.

\subsection{Empirical Illustrations with a Log-Concave Source Distribution}
\label{app:logconcave_experiments}

The simulations in Section~\ref{sec:examples} take the source distribution to be a standard Gaussian, since this is the most common choice in pushforward generative modeling. Our theory in Section~\ref{sec:gradient-bounds}, however, applies to any isotropic log-concave source distribution. Here we replicate the experiments of Section~\ref{sec:examples} with the Gaussian source replaced by a Laplace distribution, which is log-concave but not strongly log-concave. More precisely, the Laplace distribution under consideration is the product of two independent Laplace distributions with location parameter $0$ and scale parameter $1/\sqrt{2}$ (since we want its variance to be $1$ to satisfy the isotropic condition). The target families and overall procedure are otherwise identical to those of Section~\ref{sec:examples}: we consider the Student-$t$, Pareto, and Fr\'echet families with parameters controlling tail heaviness and construct the transport map coordinate-wise via the probability integral transform $T(z) = F_\mu^{-1}(F_\nu(z))$ as in Section~\ref{subsec:construction}.

We report (a) comparisons of the numerically evaluated expected squared Frobenius norm of the Jacobian against the lower bound predicted by the log-concave Poincar\'e inequality (Theorem~\ref{thm:log_concave_poincare}, part (iii) of Theorem~\ref{lem:vector_poincare}), and (b) heatmaps of the log absolute Jacobian determinant across the input space. For the Gaussian case of the Poincar\'e
inequality, the leading constant is known. For log-concave distributions, the constant in the Poincar\'e inequality is in general not known. Thus, we use a working constant of $2$ for our experiments here, which allows us to inspect how closely the bound tracks the Jacobian size as tail behavior changes over these families.

The purpose of these supplementary experiments is to verify that the qualitative behavior identified in the Gaussian case is not an artifact of the Gaussian source: heavy-tailed targets continue to force large Jacobian norms when transporting from a non-Gaussian, log-concave source, in agreement with the moment-to-Jacobian inequalities for general isotropic log-concave measures.

\begin{figure}[!htb]
    \centering
    \begin{subfigure}{0.32\textwidth}
        \includegraphics[width=\textwidth]{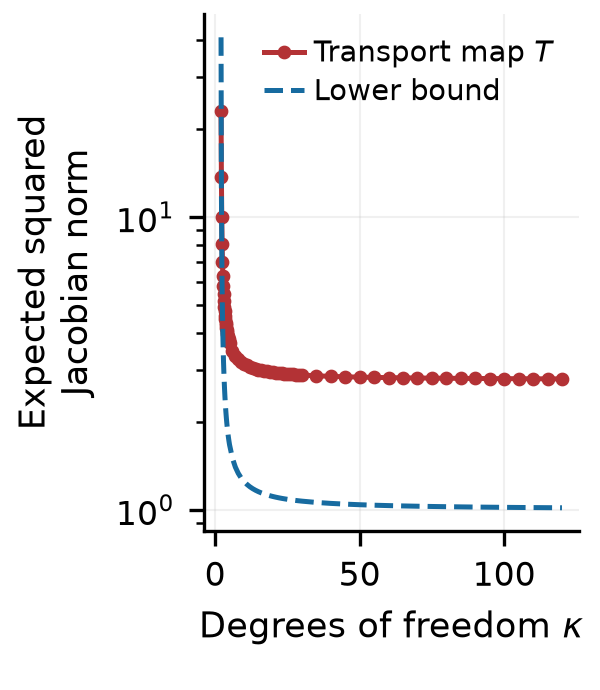}
        \caption{Student-$t$ target}
    \end{subfigure}
    \hfill
    \begin{subfigure}{0.32\textwidth}
        \includegraphics[width=\textwidth]{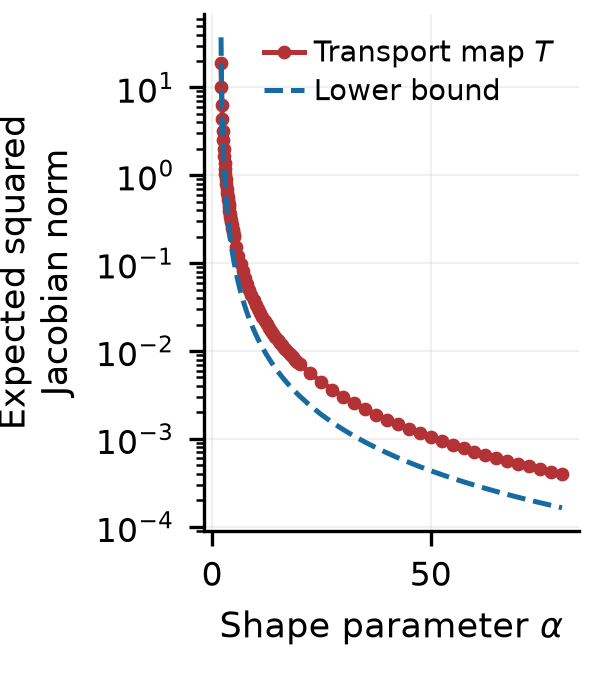}
        \caption{Pareto target}
    \end{subfigure}
    \hfill
    \begin{subfigure}{0.32\textwidth}
        \includegraphics[width=\textwidth]{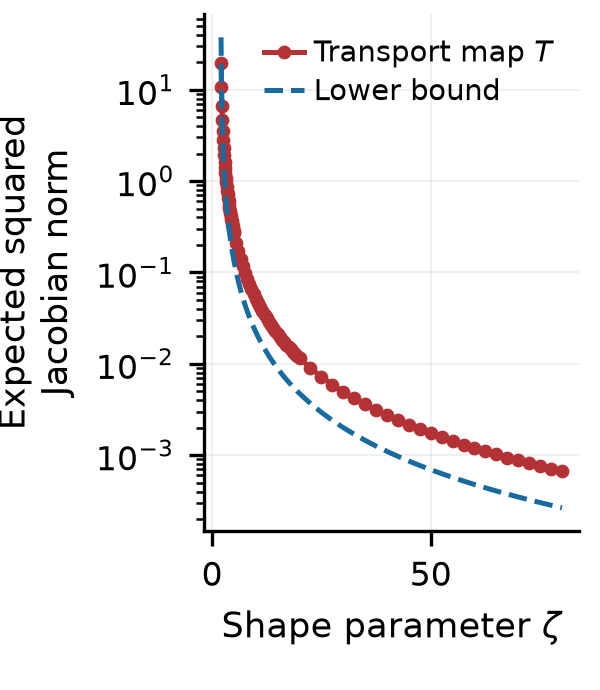}
        \caption{Fr\'echet target}
    \end{subfigure}
    \caption{Comparison of Poincar\'e lower bounds with the expected squared Frobenius norm of the transport Jacobian for a product Laplace source and heavy-tailed target distributions. The red curves with circular markers show the numerically evaluated quantities. The blue dashed curves show the lower bounds, using $C_{\mathrm{PI}}(\nu)=2$.}
    \label{fig:jacobian_vertical_logconcave}
\end{figure}

\begin{figure}[!htb]
    \centering
    \begin{subfigure}{\textwidth}
        \centering
        \includegraphics[width=\textwidth]{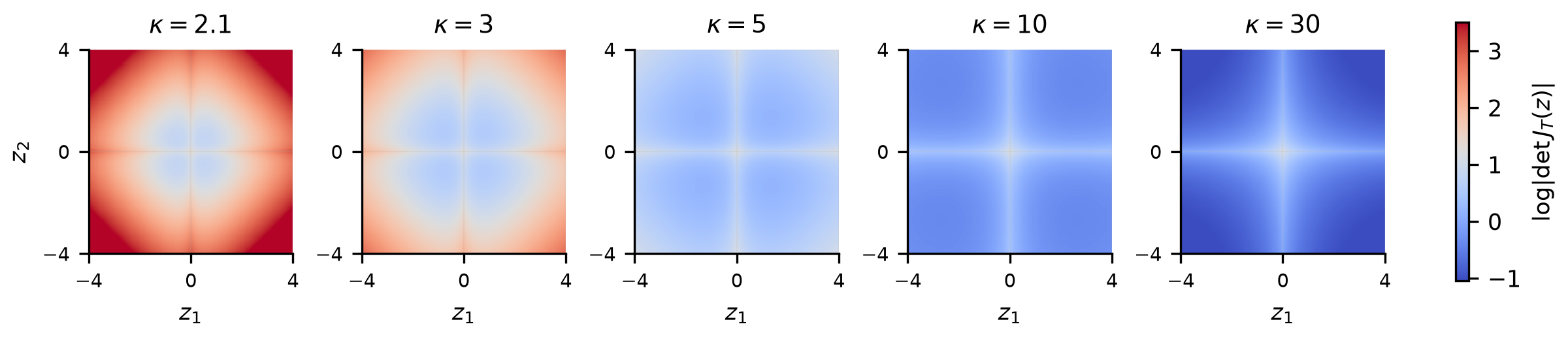}
        \caption{Student-$t$ target}
    \end{subfigure}

    \begin{subfigure}{\textwidth}
        \centering
        \includegraphics[width=\textwidth]{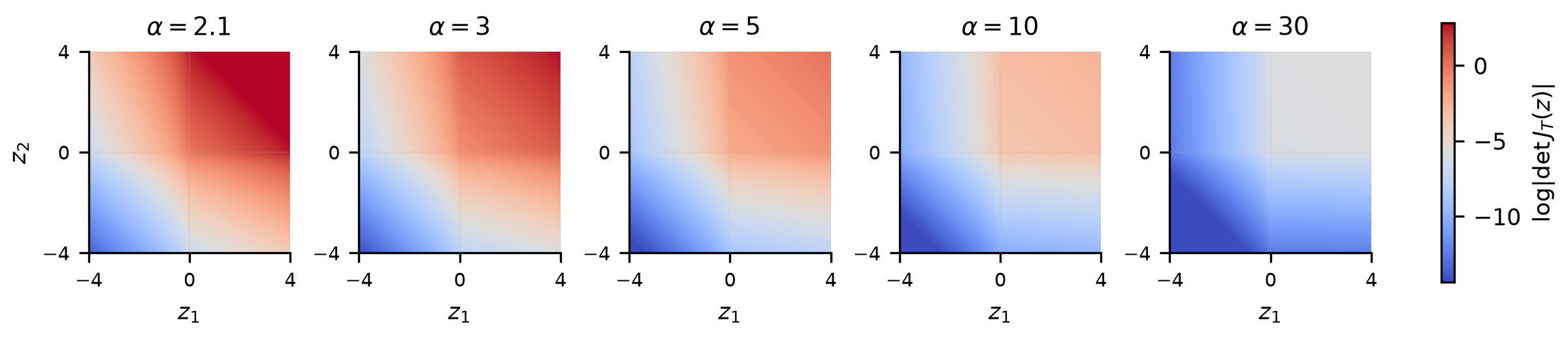}
        \caption{Pareto target}
    \end{subfigure}

    \begin{subfigure}{\textwidth}
        \centering
        \includegraphics[width=\textwidth]{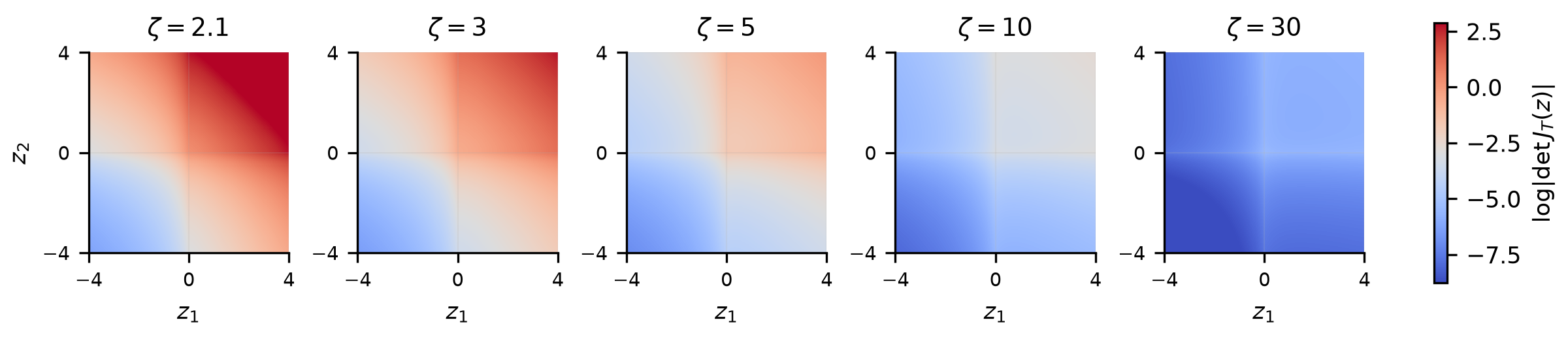}
        \caption{Fr\'echet target}
    \end{subfigure}

    \caption{Log absolute Jacobian determinant heatmaps for log-concave-to-heavy-tailed transport maps. As in the Gaussian-source case (Figure~\ref{fig:heatmap}), the magnitude of the Jacobian grows rapidly away from the center as the target becomes increasingly heavy-tailed. Note that the Pareto and Fr\'echet distributions are right-skewed and supported on the positive orthant.}
    \label{fig:heatmap_logconcave}
\end{figure}

As shown in Figure~\ref{fig:jacobian_vertical_logconcave}, the numerical estimates of the expected squared Frobenius norm continue to track the log-concave Poincar\'e lower bounds across all three target families and across the full range of tail-heaviness parameters considered. The heatmaps in Figure~\ref{fig:heatmap_logconcave} also display the same qualitative pattern as in the Gaussian-source case (Figure~\ref{fig:heatmap}): rapid Jacobian growth in regions corresponding to the heavy-tailed behavior of the target. Together, these results confirm that the moment-to-Jacobian phenomenon identified in the main text is robust to the choice of (light-tailed) log-concave source and is not specific to Gaussian sources.

\subsection{Empirical Illustrations with Correlated Heavy-Tailed Targets}
\label{app:correlated_experiments}

The simulations in Section~\ref{sec:examples} take the target distribution to be a product measure with independent coordinates, since the resulting transport map is then available coordinate-wise in closed form via the probability integral transform. Our theoretical results in Section~\ref{sec:gradient-bounds}, however, are stated in full generality and do not require the target to be a product measure. In this section, we replicate the experiments of Section~\ref{sec:examples} with the independent product target replaced by a correlated two-dimensional heavy-tailed targets. We replicate the experiments with correlated heavy-tailed targets obtained as linear transformations of independent heavy-tailed coordinates. Specifically, we first construct an independent heavy-tailed vector with Student-$t$, Pareto, or Fr\'echet coordinates as before, and then we apply a linear map $L$, chosen as the Cholesky factor of a prescribed correlation matrix.

The procedure is otherwise identical to that of Section~\ref{sec:examples}: we transport a standard Gaussian source to the correlated heavy-tailed target, vary the parameters that control tail heaviness and the correlation parameter $\rho\in\{0.5,0.9\}$, and report comparisons of the numerically evaluated expected squared Frobenius norm of the Jacobian against the lower bound predicted by the Gaussian Poincar\'e inequality (Theorem~\ref{lem:vector_poincare}, part (iii)).

These supplementary experiments verify that the qualitative behavior identified in the product-target case extends to targets with nontrivial dependence structure. Heavy-tailed correlated targets continue to force large Jacobian norms, in agreement with our inequalities, which apply to general (not necessarily product) targets.

\begin{figure}[!htb]
    \centering
    \renewcommand{\thesubfigure}{\Alph{subfigure}}
    \textbf{$\rho=0.5$}\par\smallskip
    \begin{subfigure}{0.32\textwidth}
        \includegraphics[width=\textwidth]{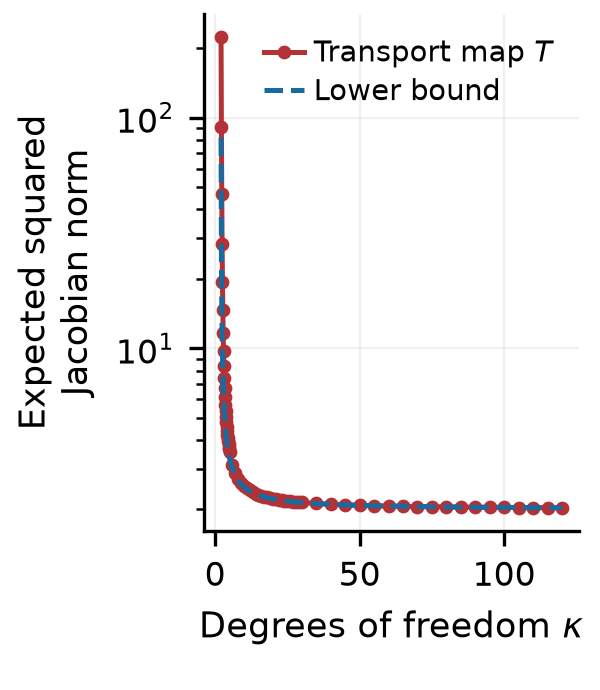}
        \caption{Student-$t$ target}
    \end{subfigure}
    \hfill
    \begin{subfigure}{0.32\textwidth}
        \includegraphics[width=\textwidth]{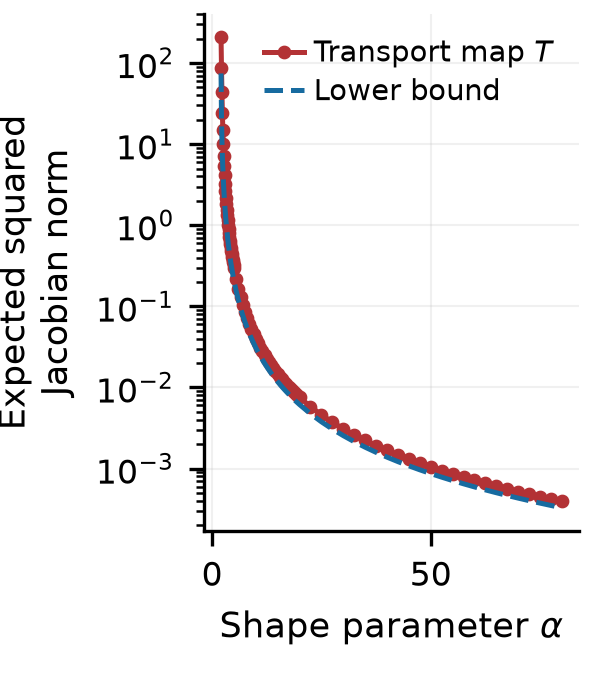}
        \caption{Pareto target}
    \end{subfigure}
    \hfill
    \begin{subfigure}{0.32\textwidth}
        \includegraphics[width=\textwidth]{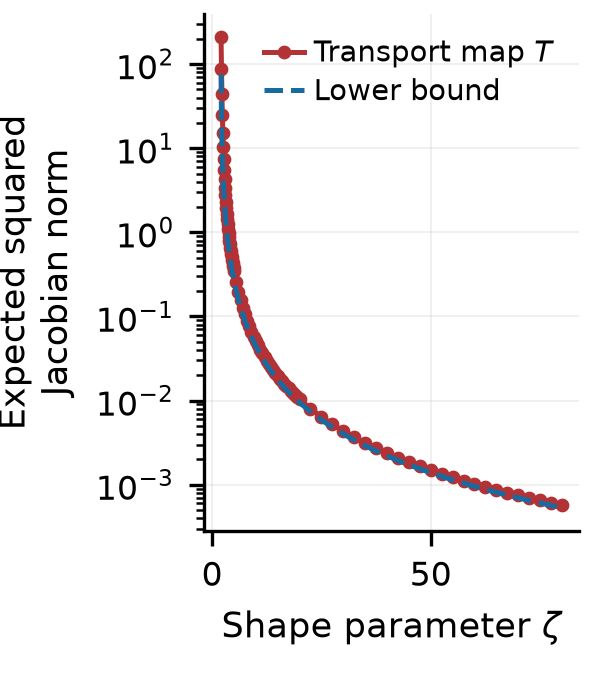}
        \caption{Fr\'echet target}
    \end{subfigure}

    \par\medskip
    \textbf{$\rho=0.9$}\par\smallskip
    \begin{subfigure}{0.32\textwidth}
        \includegraphics[width=\textwidth]{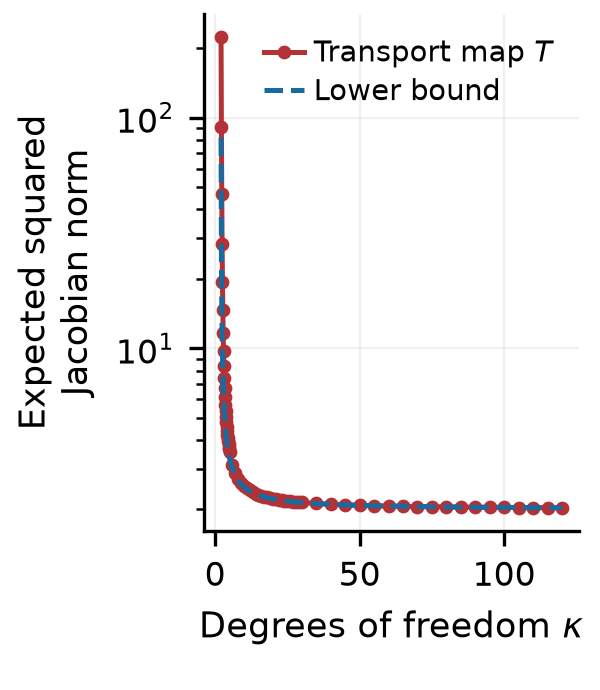}
        \caption{Student-$t$ target}
    \end{subfigure}
    \hfill
    \begin{subfigure}{0.32\textwidth}
        \includegraphics[width=\textwidth]{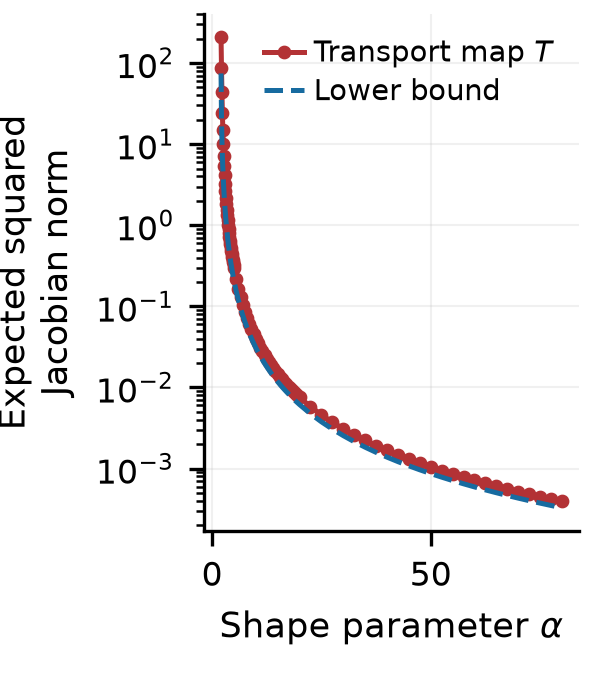}
        \caption{Pareto target}
    \end{subfigure}
    \hfill
    \begin{subfigure}{0.32\textwidth}
        \includegraphics[width=\textwidth]{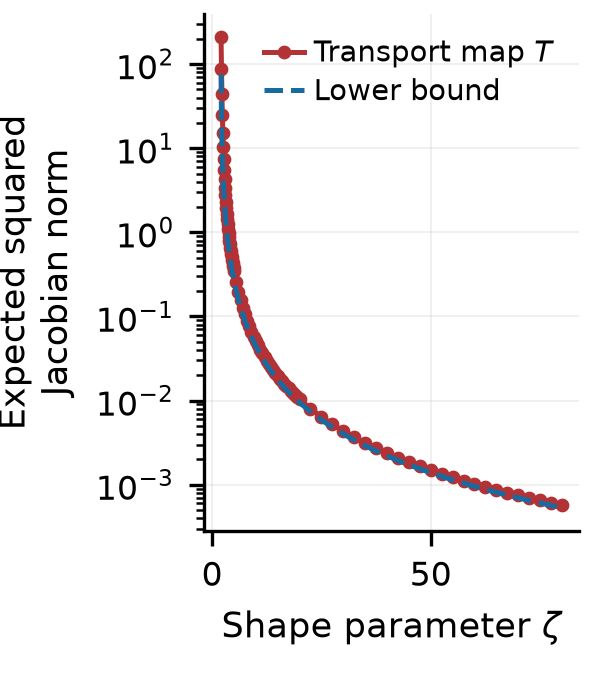}
        \caption{Fr\'echet target}
    \end{subfigure}
    \caption{Comparison of Poincar\'e lower bounds with the expected squared Frobenius norm of the transport Jacobian for correlated heavy-tailed targets under a standard Gaussian source. The top row, panels (A)--(C), uses $\rho=0.5$, and the bottom row, panels (D)--(F), uses $\rho=0.9$. Columns correspond to the Student-$t$, Pareto, and Fr\'echet target constructions. Each panel shows the numerically evaluated quantity in red and the lower bound computed from the trace of the target covariance matrix as a blue dashed curve. Compare with Figure~\ref{fig:jacobian_vertical}, where the target has independent coordinates.}
    \label{fig:jacobian_vertical_correlated}
\end{figure}

As shown in Figure~\ref{fig:jacobian_vertical_correlated}, the numerical estimates of the expected squared Frobenius norm continue to closely track the Poincar\'e-based lower bounds across all three target families when the independent heavy-tailed coordinates are linearly transformed to introduce correlation.

{\hypersetup{linkcolor=blue,filecolor=blue}
\subsection{Details of Numerical Experiments}
\label{supp:fixed-transform-experiments}
\label{cm:sec:numerics}

For the heavy-tail experiment in Figure~\ref{emp:fig:transformed-tails}, we compare the exact maps $T$ and $S=\Psi\circ T$ using a standard Gaussian source and no additional standardization. For the two-dimensional product Student, Pareto, and Fr\'echet targets, we apply $\operatorname{arsinh}(x/s)$ coordinate-wise, with $s=\sqrt\kappa$ for Student-$t$ and $s=1$ otherwise. Expectations are evaluated using Gauss--Legendre quadrature.

For the multimodal experiment in Figure~\ref{emp:fig:transformed-modes}, we use five equally weighted Gaussian components with unit standard deviations and means
\[
 m_i=D(i-3),\quad i=1,\ldots,5,\qquad D\in\{3,4,\ldots,16\}.
\]
The supplied reference component locations have the following
errors relative to the component means:
\[
(c_1-m_1,c_2-m_2,c_3-m_3,c_4-m_4,c_5-m_5)
=(0.4,-0.3,0.5,-0.4,0.2).
\]
These differences are held fixed for every value of $D$.
Throughout the study, $r=1.5$ and $h=0.4$ are fixed. No target CDF, weights, or individual component widths are supplied to the transformation.
The reference component locations are supplied synthetically to isolate the reparametrization; this study does not evaluate a procedure for locating components.

For both $g=T$ and $g=S$, we report $\E[g'(Z)^2]$ for $Z\sim N(0,1)$. Writing $p,F$ for the mixture density and CDF, and $I(u)=\varphi(\Phi^{-1}(u))$, we evaluate
\begin{equation}\label{cm:eq:quadrature}
 \E[S'(Z)^2]=\int_{\R}\frac{a(x)^2I(F(x))^2}{p(x)}\,dx.
\end{equation}
Setting $a=1$ gives the original-map quantity.

}

{\hypersetup{linkcolor=blue,filecolor=blue}
\subsection{Density-Based Fitting for Figure~\ref{fig:density-fitting}}
\label{supp:figure-one-fitting}

The heavy-tailed target has two independent standard Student-$t$ coordinates with $2.1$ degrees of freedom. The bimodal target is $\tfrac12N(m_1,\sigma^2I_2)+\tfrac12N(m_2,\sigma^2I_2)$, where $m_1=(-10,0)$, $m_2=(10,0)$, and $\sigma=0.65$. Both experiments use a standard Gaussian source on $\R^2$.

For the Student-$t$ target, apply \eqref{tf:eq:asinh} coordinate-wise with $s=\sqrt{2.1}$. For the Gaussian mixture, apply \eqref{tf:eq:component-means-map} separately to each coordinate, using the reference component locations $c_i=m_i+0.25\sigma v_i$, with $v_1=(0.8,-0.6)$ and $v_2=(-0.7,0.9)$. Set $r=1.5\sigma=0.975$ and $h=0.9\sigma=0.585$.

The normalizing flows are trained by minimizes reverse Kullback-Leibler divergence against the original target density $p$. Writing $T_\theta$ for the complete generator, $q_\theta$ for its density, and $\phi_2$ for the standard Gaussian density, the objective is
\[
 D_{\mathrm{KL}}(q_\theta\Vert p)
 =\E_{Z\sim N(0,I_2)}\!\left[
 \log\phi_2(Z)-\log|\det J_{T_\theta}(Z)|-\log p(T_\theta(Z))\right].
\]

The learned architecture is a four-layer affine-coupling RealNVP flow \citep{dinh2016density}. The network has two 32-unit $\tanh$ hidden layers. Training consists of 800 Adam updates with batch size 512 and learning rate $10^{-3}$. Heatmaps are drawn with 100,000 generated samples. Target contours are evaluated analytically. 
 
}

}
\end{document}